%% file: cvpr.tex
\documentclass[final]{cvpr}

\usepackage{times}
\usepackage{epsfig}
\usepackage{graphicx}
\usepackage{amsmath}
\usepackage{amssymb}
\usepackage{booktabs}
\usepackage{multirow}
\usepackage{amsthm}

\usepackage{sidecap}
\usepackage{subfig}

\usepackage{xr}
\usepackage[toc,page]{appendix}

\usepackage[pagebackref=true,breaklinks=true,colorlinks,bookmarks=false]{hyperref}

\def\cvprPaperID{1262} 
\def\confYear{CVPR 2021}

\begin{document}

\title{Generative Interventions for Causal Learning}

\newcommand*\samethanks[1][\value{footnote}]{\footnotemark[#1]}

\author{Chengzhi Mao \\
Columbia University\\
{\tt\small mcz@cs.columbia.edu}
\and
Augustine Cha\thanks{Equal Contribution. Order by coin flip.}\\
Columbia University\\
{\tt\small ac4612@columbia.edu}
\and
Amogh Gupta\samethanks \\
Columbia University\\
{\tt\small ag4202@columbia.edu}
\and
Hao Wang\\
Rutgers University\\
{\tt\small hoguewang@gmail.com}
\and
Junfeng Yang\\
Columbia University\\
{\tt\small junfeng@cs.columbia.edu}
\and
Carl Vondrick\\
Columbia University\\
{\tt\small vondrick@cs.columbia.edu}
}

\maketitle

\input{source/def}

\input{source/abstract}

\input{source/intro}

\input{source/relatedwork}

\input{source/causal_graph}

\input{source/method}

\input{source/experiment}

\input{source/conclusion}

{\small
\textbf{Acknowledgments:} This research is based on work partially supported by NSF NRI Award \#1925157, the DARPA SAIL-ON program under PTE Federal Award No.\ W911NF2020009, an Amazon Research Gift, NSF grant CNS-15-64055, ONR grants N00014-16-1- 2263 and N00014-17-1-2788, a JP Morgan Faculty Research Award, a DiDi Faculty Research Award, a Google Cloud grant, and an Amazon Web Services grant. We thank NVidia for GPU donations. The views and conclusions contained herein are those of the authors and should not be interpreted as necessarily representing the official policies, either expressed or implied, of the U.S. Government.
}

{
\small
\bibliography{reference}
}

\externaldocument[S-]{supplementary}

\input{supp/supp_main.tex}

\end{document}

%% file: source/def.tex
\def\Blue{\color{blue}}
\def\Purple{\color{purple}}

\def\A{{\bf A}}
\def\a{{\bf a}}
\def\B{{\bf B}}
\def\b{{\bf b}}
\def\C{{\bf C}}
\def\c{{\bf c}}
\def\D{{\bf D}}
\def\d{{\bf d}}
\def\E{{\bf E}}
\def\e{{\bf e}}
\def\f{{\bf f}}
\def\F{{\bf F}}
\def\K{{\bf K}}
\def\k{{\bf k}}
\def\L{{\bf L}}
\def\H{{\bf H}}
\def\h{{\bf h}}
\def\G{{\bf G}}
\def\g{{\bf g}}
\def\I{{\bf I}}
\def\R{{\bf R}}
\def\X{{\bf X}}
\def\Y{{\bf Y}}
\def\OO{{\bf O}}
\def\oo{{\bf o}}
\def\P{{\bf P}}
\def\Q{{\bf Q}}
\def\r{{\bf r}}
\def\s{{\bf s}}
\def\S{{\bf S}}
\def\t{{\bf t}}
\def\T{{\bf T}}
\def\x{{\bf x}}
\def\y{{\bf y}}
\def\z{{\bf z}}
\def\Z{{\bf Z}}
\def\M{{\bf M}}
\def\m{{\bf m}}
\def\n{{\bf n}}
\def\U{{\bf U}}
\def\u{{\bf u}}
\def\V{{\bf V}}
\def\v{{\bf v}}
\def\W{{\bf W}}
\def\w{{\bf w}}
\def\0{{\bf 0}}
\def\1{{\bf 1}}
\def\N{{\bf N}}

\def\AM{{\mathcal A}}
\def\EM{{\mathcal E}}
\def\FM{{\mathcal F}}
\def\TM{{\mathcal T}}
\def\UM{{\mathcal U}}
\def\XM{{\mathcal X}}
\def\YM{{\mathcal Y}}
\def\NM{{\mathcal N}}
\def\OM{{\mathcal O}}
\def\IM{{\mathcal I}}
\def\GM{{\mathcal G}}
\def\PM{{\mathcal P}}
\def\LM{{\mathcal L}}
\def\MM{{\mathcal M}}
\def\DM{{\mathcal D}}
\def\SM{{\mathcal S}}
\def\RB{{\mathbb R}}
\def\EB{{\mathbb E}}

\def\tx{\tilde{\bf x}}
\def\ty{\tilde{\bf y}}
\def\tz{\tilde{\bf z}}
\def\hd{\hat{d}}
\def\HD{\hat{\bf D}}
\def\hx{\hat{\bf x}}
\def\hR{\hat{R}}

\def\Ome{\mbox{\boldmath$\omega$\unboldmath}}
\def\bet{\mbox{\boldmath$\beta$\unboldmath}}
\def\et{\mbox{\boldmath$\eta$\unboldmath}}
\def\ep{\mbox{\boldmath$\epsilon$\unboldmath}}
\def\ph{\mbox{\boldmath$\phi$\unboldmath}}
\def\Pii{\mbox{\boldmath$\Pi$\unboldmath}}
\def\pii{\mbox{\boldmath$\pi$\unboldmath}}
\def\Ph{\mbox{\boldmath$\Phi$\unboldmath}}
\def\Ps{\mbox{\boldmath$\Psi$\unboldmath}}
\def\pss{\mbox{\boldmath$\psi$\unboldmath}}
\def\tha{\mbox{\boldmath$\theta$\unboldmath}}
\def\Tha{\mbox{\boldmath$\Theta$\unboldmath}}
\def\muu{\mbox{\boldmath$\mu$\unboldmath}}
\def\Si{\mbox{\boldmath$\Sigma$\unboldmath}}
\def\Gam{\mbox{\boldmath$\Gamma$\unboldmath}}
\def\gamm{\mbox{\boldmath$\gamma$\unboldmath}}
\def\Lam{\mbox{\boldmath$\Lambda$\unboldmath}}
\def\De{\mbox{\boldmath$\Delta$\unboldmath}}
\def\vps{\mbox{\boldmath$\varepsilon$\unboldmath}}
\def\Up{\mbox{\boldmath$\Upsilon$\unboldmath}}
\def\Lap{\mbox{\boldmath$\LM$\unboldmath}}
\newcommand{\ti}[1]{\tilde{#1}}

\def\tr{\mathrm{tr}}
\def\etr{\mathrm{etr}}
\def\etal{{\em et al.\/}\,}
\newcommand{\indep}{{\;\bot\!\!\!\!\!\!\bot\;}}
\def\argmax{\mathop{\rm argmax}}
\def\argmin{\mathop{\rm argmin}}
\def\vec{\text{vec}}
\def\cov{\text{cov}}
\def\dg{\text{diag}}

\newtheorem{remark}{Remark}
\newtheorem{theorem}{Theorem}
\newtheorem{lemma}{Lemma}
\newtheorem{definition}{Definition}

\newtheorem{proposition}{Proposition}

\let\oldemptyset\emptyset
\let\emptyset\varnothing

\newcommand\independent{\protect\mathpalette{\protect\independenT}{\perp}}
\def\independenT#1#2{\mathrel{\rlap{$#1#2$}\mkern2mu{#1#2}}}

%% file: source/abstract.tex
\begin{abstract}
We introduce a framework for learning robust visual representations that generalize to new viewpoints, backgrounds, and scene contexts. Discriminative models often learn naturally occurring spurious correlations, which cause them to fail on images outside of the training distribution. In this paper, we show that we can steer generative models to manufacture interventions on features caused by confounding factors. Experiments, visualizations, and theoretical results show this method learns robust representations  more consistent with the underlying causal relationships. Our approach improves performance on multiple datasets demanding out-of-distribution generalization, and we demonstrate state-of-the-art performance generalizing from ImageNet to ObjectNet dataset. 
\end{abstract}

%% file: source/intro.tex
\section{Introduction}

Visual recognition today is governed by empirical risk minimization (ERM), which bounds the generalization error when the training and testing distributions match \cite{ERM}.  When training sets cover all factors of variation, such as background context or camera viewpoints, discriminative models learn invariances and predict object category labels with the right cause \cite{pearl}.
However, the visual world is vast and naturally open. Collecting a representative, balanced dataset is difficult and, in some cases, impossible because the world can unpredictably change after learning.

Directly optimizing the empirical risk is prone to learning unstable spurious correlations that do not respect the underlying causal structure~\cite{causalfeature, arjovsky2019invariant,ilse2020designing, causalexplain, interventcausal, CIinvariant}.  Figure \ref{fig:teaser} illustrates the issue succinctly. In natural images, the object of interest and the scene context have confounding factors, creating spurious correlations. 
For example, ladle (the object of interest) often has a hand holding it (the scene context), but there is no causal relation between them. 
Several studies have exposed this challenge by demonstrating substantial performance degradation when the confounding bias no longer holds at testing time \cite{ImageNetOverfit, imagenetbiased}. For example, the ObjectNet \cite{Objectnet} dataset removes several common spurious correlations from the test set, causing the performance of state-of-the-art models to deteriorate by 40\% compared to the ImageNet validation set.

A promising direction for fortifying visual recognition is to learn \emph{causal} representations (see \cite{scholkopf2019causality} for an excellent overview). If representations are able to identify the causal mechanism between the image features and the category labels, then robust generalization is possible.
While the traditional approach to establish causality is through randomized control trials or interventions, natural images are passively collected, preventing the use of such procedures. 

\setcounter{footnote}{-1}
\begin{figure}
    \centering
    \includegraphics[width=\linewidth]{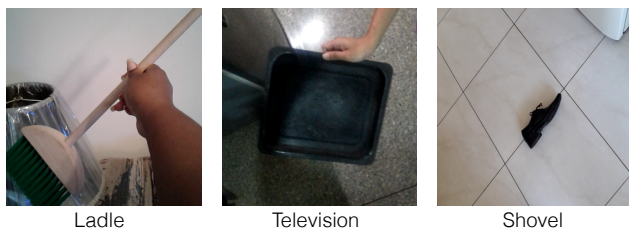}
    \caption[]{Top predictions from a state-of-the-art ImageNet classifier \cite{He_2016}. The model uses spurious correlations (scene contexts, viewpoints, and backgrounds), leading to incorrect predictions.\footnotemark\ In this paper, we introduce a method to learn causal visual features that improve robustness of visual recognition models.}
    \label{fig:teaser}
    \vspace{-5mm}
\end{figure}
\footnotetext{The correct categories are clearly a broom, a tray, and a shoe.}


\begin{figure*}[t]
  \centering
  \vspace{-4mm}
  \includegraphics[width=0.95\textwidth]{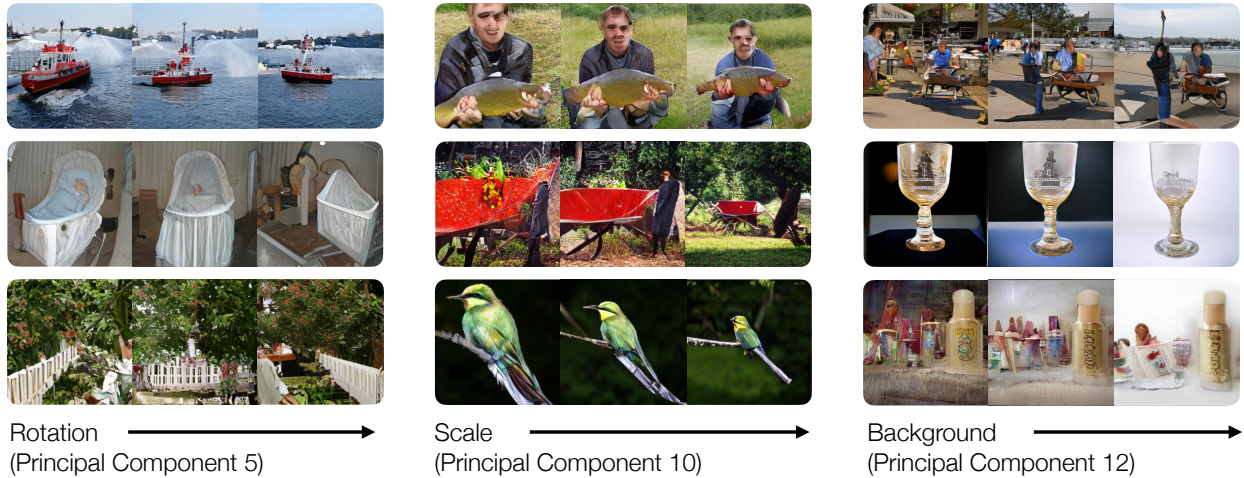}
  \caption{Generative adversarial networks are steerable \cite{hrknen2020ganspace, jahanian2019steerability}, allowing us to manipulate images and construct interventions on nuisances. The transformations transfer across categories. Each column in the figure presents images with one consistent intervention direction.\vspace{-5mm}}
  \label{fig:equi}
\end{figure*}

This paper introduces a framework for learning causal visual representations with natural images.
Our approach is based on the observation that generative models quantify nuisance variables \cite{hrknen2020ganspace, jahanian2019steerability}, such as viewpoint or background. We present a causal graph that models both robust features and spurious features during image recognition. Crucially, our formulation shows how to learn causal features by steering generative models to perform interventions on realistic images, simulating manipulations to the camera and scene that remove spurious correlations. 
As our approach is model-agnostic, we are able to learn robust representations for any state-of-the-art computer vision model.

Our empirical and theoretical results show that our approach learns representations that regard causal structures.
While just sampling from generative models will replicate the same training set biases, steering the generative models allows us to reduce the bias, which we show is critical for performance. On ImageNet-C \cite{imgnet-C} benchmark, we surpass established methods by up to 12\%, which shows that our method helps discriminate based on the causal features. 
Our approach also demonstrates the state-of-the-art performance on the new ObjectNet dataset \cite{Objectnet}. We obtain $39.3\%$ top-1 accuracy with ResNet152 \cite{He_2016}, which is over 9\% gain over the published ObjectNet benchmark \cite{Objectnet} while maintaining accuracy on ImageNet and ImageNet-V2 \cite{ImageNetOverfit}. Our code is available at \url{https://github.com/cvlab-columbia/GenInt}. 



%

%% file: source/relatedwork.tex
\section{Related Work}


\textbf{Data augmentation:} Data augmentation often helps learn robust image classifiers. Most existing data augmentations use lower-level transformations \cite{AlexNet, tian2020rethinking}, such as rotate, contrast, brightness, and shear. Auto-data augmentation \cite{cubuk2018autoaugment, zhang2019adversarialautoaug} uses reinforcement learning to optimize the combination of those lower-level transformations. Other work, such as \emph{cutout}~\cite{maskout} and \emph{mixup}~\cite{zhang2017mixup}, develops new augmentation strategies towards improved generalization.
\cite{perez2017effectivenessaug, cyclegan-aug, imagenetbiased} explored style transfer to augment the training data, however, the transformations for training are limited to texture and color change. Adversarial training, where images are augmented by adding imperceptible adversarial noise, can also train robust models \cite{xie2019adversarial}. However, both adversarial training \cite{xie2019adversarial} and auto augmentation \cite{cubuk2018autoaugment, zhang2019adversarialautoaug} introduce up to three orders of magnitude of computational overhead. In addition, none of the above methods can do high-level transformations such as changing the viewpoint or background \cite{Objectnet}, while our generative interventions can. 
Our method fundamentally differs from prior data augmentation methods because it learns a robust model by estimating the causal effects via generative interventions. Our method not only eliminates spurious correlations more than data augmentations, but also theoretically produces a tighter causal effect bound. 



\textbf{Causal Models:} Causal image classifiers generalize well despite environmental changes because they are invariant to the nuisances caused by the confounding factors \cite{arjovsky2019invariant}. A large body of work studies how to acquire causal effects from a combination of association levels and intervention levels \cite{causaladaptation, causalfeature, nair2019causal}. Ideally, we can learn an invariant representation across different environments and nuisances \cite{arjovsky2019invariant, CIinvariant} while maintaining the causal information \cite{aless2017emergence}. While structural risk minimization, such as regularization \cite{janzing2019causalregularization}, can also promote a model's causality, this paper focuses on training models under ERM~\cite{ERM}. 




\textbf{Generative Models:} 
Our work leverages recent advances in deep generative models \cite{goodfellow2014generative, kingma2013autoencoding, Jaakkola, BigGAN, razavi2019vqvae}. Deep generative models capture the joint distribution of the data, which can complement discriminative models~\cite{Jaakkola, radford2015unsupervised}. Prior work has explored adding data directly sampled from a deep generator to the original training data to improve classification accuracy on ImageNet \cite{CAS}. We denote it as GAN Augmentation in this paper.
Other works improved classification accuracy under imbalanced and insufficient data by oversampling through a deep generator \cite{GANaug, GanLiver, lowshotgan, antoniou2017gandataaug}. However, sampling without intervention, the augmented data still follows the same training joint distribution, where unobserved confounding bias will continue to contaminate the generated data. Thus, the resulting models still fail to generalize once the spurious correlations changed. Ideally, we want to generate data independent of the spurious correlations while holding the object's causal features fixed.


Recent works analyzing deep generative models show that different variations, such as viewpoints and background, are automatically learned \cite{jahanian2019steerability, hrknen2020ganspace}. We leverage deep generative models for constructing interventions in realistic visual data. Our work randomizes a large class of steerable variations, which shifts the observed data distribution to be independent of the confounding bias further. 
Our approach tends to manipulate high-level transformations orthogonal to traditional data augmentation strategies \cite{augsurvey, tian2020rethinking, AlexNet}, and we obtain additional performance gains by combining them.


\textbf{Domain Adaptation:} Our goal is to train robust models that generalize to unforeseen data. Accessing the test data distribution, even unlabeled, could lead to overfitting and fail to measure the true generalization. Our work thus is trained with no access to the test data. Our setting is consistent with ObjectNet's policy prohibiting any form of learning on its test set \cite{Objectnet}, and ImageNet-C's policy discouraging training on the tested corruptions. On the other hand, domain adaptation \cite{Hoffman_cycada2017, aligndomain_2018_CVPR, wang2020continuously} needs access to the distributions of both the source domain and the target domain, which conflicts with our setting.


%% file: source/causal_graph.tex
\section{Causal Analysis}
We quantify nuisances via generative models and propose the corresponding causal graph. We show how to train causal models via intervention on the nuisance factors. We theoretically show sufficient conditions for intervention strategy selection that promote causal learning.

\begin{figure}[t]
  \centering
  \vspace{-5mm}
  \includegraphics[width=0.48\textwidth]{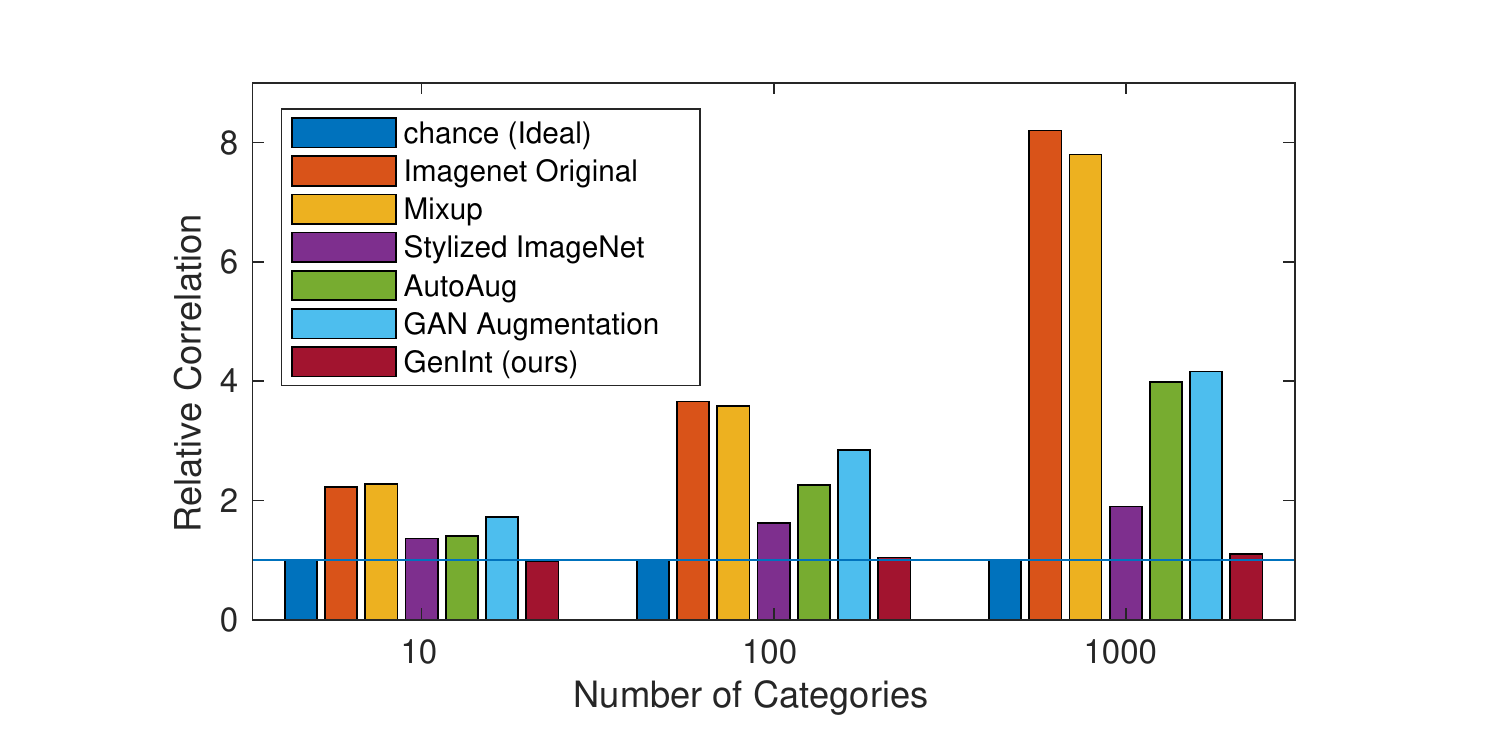}
  \vspace{-5mm}
  \caption{Do unwanted correlations exist between the nuisance factors (e.g. backgrounds, viewpoint) and labels on ImageNet? 
  We measure correlation (y-axis) via how many times the classification accuracy is better than chance on the ImageNet validation set. 
  The x-axis denotes the number of categories we select for prediction. To train causal models, nuisance factors should not be predictable for labels (chance). Our generative interventions (GenInt) reduce the unwanted correlations from the data better than existing data augmentation strategies \cite{zhang2017mixup, cubuk2018autoaugment, imagenetbiased, CAS}.}
   \vspace{-5mm}
  \label{fig:spurious_measure}
\end{figure}

\subsection{Correlation Analysis}\label{sec:corr}
  

Nuisance factors do not cause the object label. If there is a correlation between the nuisance factors and the label in data, we cannot learn causal classifiers. While identifying such correlations is crucial, they are hard to quantify on large, real-world vision datasets, because nuisance factors such as viewpoint and backgrounds, are difficult and expensive to measure in natural images. 


We propose to measure such nuisance factors via intervening on the conditional generative models.
Prior work \cite{hrknen2020ganspace, jahanian2019steerability} shows that nuisance transformations automatically emerge in generative models (Figure \ref{fig:equi}), which enables constructing desired nuisances via intervention. Given a category $\y$ and random noise vector $\h_0$, we first generate an exemplar image $\x = G(\h_0,\y)$. We then conduct intervention $\z$ to get the intervened noise vector $\h_0^*$, and the intervened image $\x^*= G(\h_0^*,\y)$, which corresponds to changing the viewpoints, backgrounds, and scene context of the exemplar. We thus get data with both image $\x^*$ and the corresponding nuisance manipulation $\z$. Implementation details are in the supplementary.  


We train a model that predicts the nuisances $\z$ from input image $\x^*$. This model can then predict nuisances $\z$ from natural images $\x$. We read out the correlation between the nuisance $\z$ and label $\y$ by training a fully-connected classifier with input $\z$ and output $\y$. We measure the correlations via the times the classifier outperforms random. Generative models may capture only a subset of the nuisances, thus our estimated correlations are lower bounds. The true correlations maybe even more significant.

In Figure \ref{fig:spurious_measure}, the training data of five established methods \cite{He_2016, zhang2017mixup, cubuk2018autoaugment, imagenetbiased, CAS} contains strong correlations that are undesirable. On the original ImageNet data, the undesirable correlation in the data is up to 8 times larger than chance. Our generative interventions reduce the unwanted correlations from the data significantly, naturally leading to robust classifiers that use the right cause.


\subsection{Causal Graph}

 \begin{figure}[t]
\centering
\vspace{-1mm}
\subfloat{\label{scm2}\includegraphics[width=0.47\textwidth, trim={0cm 0 0cm 0},clip]{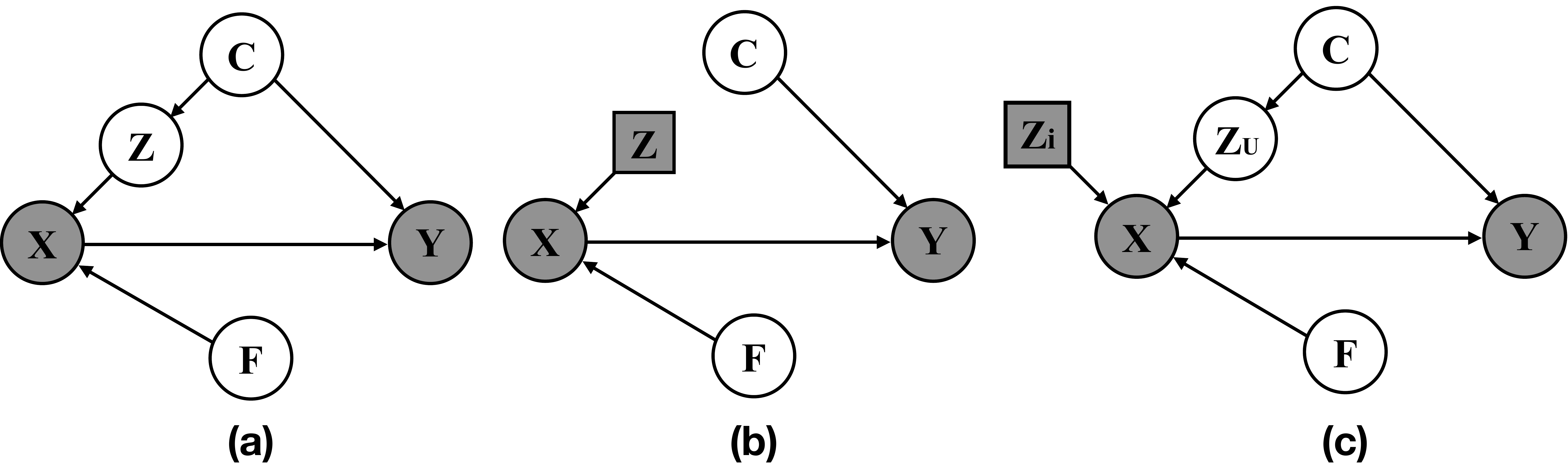}}
\caption{Causal graph for image classification. Gray variables are observed. (\textbf{a}) $F$ is the variable that generates the object features. The unobserved confounder $C$ causes both the  background features $Z$ and label $Y$, which creates a spurious correlation between the image $X$ and label $Y$. (\textbf{b}) An ideal intervention blocks the backdoor path from $Z$ to $C$, which produces causal models. (\textbf{c}) In practice, we cannot guarantee to intervene on all the $Z$ variables. However, by properly intervening on even a small set of nuisance factors $Z_i$, the confounding bias of the observed distribution is mitigated, which is theoretically proven by Theorem \ref{th:2IV}.}
  \label{fig:causalgraph}
  \vspace{-5mm}
\end{figure}
We build our causal graph based on the correlation analysis. We know that nuisances do not cause the label (context `hand' does not cause the category `ladle'), and there is no additional common outcome variable (collider) in our correlation prediction. If the correlation between the nuisances and the label is not chance, then there exists a confounder $C$ that causes both $Z$ and $Y$.

Figure \ref{fig:causalgraph}(a) shows our causal graph for image recognition. We denote the unobserved confounder as $C$, which produces nuisance factors $Z$, and the corresponding annotated categorical label $Y$. $Z$ produces the nuisance features $X_Z$ in images. There is another variable $F$ that generates the core object features $X_F$, which together with $X_Z$ constructs the pixels of a natural image $X$. There is no direct arrow from $F$ to $Y$ since $Y \independent F | X$, i.e., image $X$ contains all the features for predicting $Y$. We can observe only $X$ but not $X_Z$ or $Z_F$ separately.  We draw a causal arrow from $X$ to $Y$. Since nuisances $Z$ are spuriously correlated to the label but not causing the label $Y$, classifiers are not causal if they predict $Y$ from the nuisances $Z$ better than chance. Note that ``while a directed path is necessary for a total causal effect, it is not sufficient \cite{elements_causal}.'' Thus, though there is a path $Z \rightarrow X \rightarrow Y$, $Z$ does not cause $Y$.



 
\subsection{Causal Discriminative Model}\label{sec:dis}

\emph{Generative interventions help in eliminating spurious correlations (Figure~\ref{fig:spurious_measure} and Section~\ref{sec:corr}), leading to better generalization.} We denote the causality from $X$ to $Y$ to be $P(\y|do(\x))$, which is the treatment effect of an input image $X$ on label $Y$.
To capture the right cause via correlations learned by empirical risk minimization, we need to construct data such that $P(Y|do(X)) = P(Y|X)$. 





Natural images are often biased by unobserved confounding factors that are common causes to both the image $X$ and the label $Y$. A passively collected vision dataset only enables us to observe the variables $X$ and $Y$. Theoretically, we cannot identify the causal effect $P(Y|do(X))$ in Figure \ref{fig:causalgraph}(a) with only the observed joint distribution $P(X,Y)$ because there is an unobserved common cause.




We thus want to collect faithful data independent of the confounding bias, so that we can identify the causal effect with only the observed data. We need to intervene on the data-generation process for the nuisances $Z$ to be independent to the confounders, while keeping the core object features $F$ unchanged. In the physical world, such interventions correspond to actively manipulating the camera or objects in the scene. In our paper, we perform such interventions via steering the generative models. The outcome of this intervention on $Z$ is visualized in Figure \ref{fig:causalgraph}(b), which manipulates the causal graph such that dependencies arriving at $Z$ are removed. Removing the backdoor, the correlation is now equal to the causality, i.e.,  $P(Y|X) = P(Y|do(X))$. While this result is intuitive, performing perfect intervention in practice is challenging due to the complexity of the natural image distribution. 

\subsection{Causal Effect Bound}
Imperfect interventions can eliminate only some spurious correlations. Though it is theoretically impossible to calculate the exact causal effect $P(\y|do(\x))$ when spurious correlations are not totally removed, we can still estimate the lower and upper bound for $P(\y|do(\x))$.

 Given the observed joint distribution $P(\x,\y)$, Pearl \cite{pearl} identified that $P(\y|do(\x))$ can be bounded by $P(\x,\y) \le P(\y|do(\x)) \le P(\x,\y)+1-P(\x)$, which can be estimated by existing discriminative models without interventions. 
 
 Prior work augments the data by sampling from the GANs without explicit intervention~\cite{lowshotgan, antoniou2017gandataaug, GANaug, GanLiver}, which will yield the same causal bound as the original data. Since GANs capture the same distribution as the observational training set, the spurious correlations remain the same. The sampled transformations $Z$ in Figure \ref{fig:causalgraph} (a) are still dependent on the confounders $C$. Thus, augmenting training data with GANs \cite{CAS}, without intervention is not an effective algorithm for causal identification. 

In this paper, we aim to identify a tighter causal effect bound for $P(\y|do(\x))$ using generative interventions. This is desirable for robustness because it removes or reduces the overlap between the causal intervals, promoting causal predictions. 
Section~\ref{sec:dis} establishes that perfect interventions eliminate all spurious correlation and leads to better generalization. In practice, our generative interventions may only eliminate a subset of spurious correlations  $Z_i$, while other nuisances $Z_U$ remain unobserved and untouched. The next question is then: \emph{what generative intervention strategy is optimal for tightening the causal effect bound?} We derive the following theory:




\begin{theorem}[\textbf{Effective Intervention Strategy}]\label{th:2IV}
We denote the images as $x$. The causal bound under intervention $z_i$ is thus $P(\y,\x|\z_i) \leq P(\y|do(\x)) \leq P(\y,\x|\z_i) + 1 - P(\x|\z_i)$. For two intervention strategies $\z_1$ and $\z_2$,  $\z_1 \subset \z, \z_2 \subset \z$, if  $P(\x|\z_1) > P(\x|\z_2)$, then $\z_1$ is more effective for causal identification. 
\end{theorem}

\begin{proof}
Figure \ref{fig:causalgraph}(c) shows the causal graph after intervention $Z_i$, where $Z_i \independent Y | X$.  We add and remove the same term $\sum_c P(\y,\x,\c|\z_i)$:
\begingroup\makeatletter\def\f@size{9}\check@mathfonts
\begin{align*}
\begin{split}
& P(\y|do(\x)) = \sum_c P(\y|\x,\z_i,\c)P(\c) \quad\;\;\;  \text{(Backdoor Criteria)} \\
&= \sum_\c P(\y,\x,\c|\z_i) + \sum_\c{P(\y|\x,\z_i,\c)(P(\c)-P(\x,\c|\z_i))}
\end{split}
\end{align*}
\endgroup
Since $0 \leq P(\y|\x,\z_i,\c) \leq 1$, we have the lower and upper bounds. 
We denote $\delta_1 = P(\x|\z_1) - P(\x|\z_2)$, thus  $\delta_1>0$. In the causal graph (Figure \ref{fig:causalgraph}(c)), since we intervene on $\z_i$, all incoming edges  to $\z_i$ are removed; 
we then have $\z_i \independent \y | \x$ and $P(\x,\y|\z_i) = P(\y|\x,\z_i)P(\x|\z_i)= P(\y|\x)P(\x|\z_i)$. Therefore $\delta_2=P(\x,\y|\z_1)-P(\x,\y|\z_2)=\delta_1 \cdot P(\y|\x)$. 
Since apparently $0<P(\y|\x)<1$, we have that $0<\delta_2<\delta_1$. Thus we obtain $[P(\y,\x|\z_1), P(\y,\x|\z_1) + 1 - P(\x|\z_1)] \subset [P(\y,\x|\z_2), P(\y,\x|\z_2) + 1 - P(\x|\z_2)]$, which means the intervention $\z_1$ results in a tighter causal effect bound.
\end{proof}
%

Our theorem shows that: the optimal intervention strategy should maximize $P(\x|\z)$, which will tighten the causal effect bound $P(\y|do(\x))$. Also, the intervention strategy should be identically selected across all categories, so that they are independent of the confounding bias. While there are different choices of intervening on the generative model to create independence, we empirically select our generative intervention strategy that increases $P(\x|\z)$, which we will discuss in Section \ref{sec:int_effect}.





%% file: source/method.tex
\vspace{-3mm}
\section{Method}
\label{sec:method}
We show how deep generative models can be used to construct interventions on the spuriously correlated features in the causal graph. We combine these results to develop a practical framework for robust learning. 

\subsection{Learning Objective}

We minimize the following training loss on our intervened data:
\begin{equation}
\begin{split}
    \mathcal{L} = &  \mathcal{L}_{e}(\phi(\X),\Y) +  \lambda_1 \mathcal{L}_{e}(\phi(\X_{int}),\Y') \\
    & + \lambda_2 \mathcal{L}_{e}(\phi(\X_{itr}),\Y'')
    \end{split}
\end{equation}
where $\mathcal{L}_{e}$ denotes the standard cross entropy loss and $\lambda_i \in \mathbb{R}$ are hyper-parameters controlling training data choice.  We denote the original data matrix as $\X$ with target labels $\Y$; the generated data matrix as $\X_{int}$ (Section~\ref{method:genint}) with target labels $\Y'$; the transfered data as $\X_{itr}$ (Section~\ref{sec:transfer}) with target labels $\Y''$; and the discriminative classifier as $\phi$. 

The last two terms of this objective are the interventions. In the remainder of this section, we present two different ways of constructing these interventions.

\subsection{Generative Interventions}\label{method:genint}


We construct interventions using conditional generative adversarial networks (CGAN). We denote the $i$-th layer's hidden representation as $\h_i$. CGAN learns the mapping $\x = G(\h_0,\y)$, where $\h_0 \sim \mathcal{N}(0, I)$ is the input noise, $\y$ is the label, and $\x$ is a generated image of class $\y$ that lies in the natural image distribution. CGANs are trained on the joint data distribution $P(\x,\y)$. While we can use any type of CGANs, we select BigGAN \cite{BigGAN} in this paper since it produces highly realistic images. In addition, generative models learn a latent representation $\h_i$ equivariant to a large class of visual transformations and independent of the object category \cite{hrknen2020ganspace, jahanian2019steerability}, allowing for controlled visual manipulation.
For example, GANSpace \cite{hrknen2020ganspace} showed that the principal components of $\h_i$ correspond to visual transformations over camera extrinsics and scene properties. The same perturbations in the latent space will produce the same visual transformations across different categories. Figure~\ref{fig:equi} visualizes this steerability for a few samples and different transformations. 
This property enables us to construct a new training distribution, where the nuisance features $Z$ are not affected by the confounders. 


Our generative intervention strategy follows the GANSpace \cite{hrknen2020ganspace} method, which empirically steers the GAN with transformations independent of the categories. It contains three factors: the truncation value, the transformation type, and the transformation scale. The input noise $\h_0$ is sampled from Gaussian noise truncated by value $t$ \cite{BigGAN}.
We define the transformations to be along the $j$-th principal directions $\r_j$ in the feature space \cite{hrknen2020ganspace}, which are orthogonal and captures the major variations of the data. We select the top-$k$ significant ones $\{\r_1, \r_2, ..., \r_k\}$ as the intervention directions. We then intervene along the selected directions with a uniformly sampled step size $s'$ from a range $[-s, s]$. We intervene on the generator's intermediate layers with $\h_i^* = \h_i + \sigma s' \r_j - \mu $, where $\h_i^*$ are the features at layer $i$ after interventions, $\sigma$ is the standard deviation of noise on direction $\r$, and $\mu$ is the offset term. After the intervention, we follow the method in GANSpace  \cite{hrknen2020ganspace} to recover $\h_0^*$ with regression and generate the new image $\x^* = G(\h_0^*, \y)$.
 Using conditional generative models, we produce the causal features $X_F$ by specifying the category. Our intervention removes the incoming edge from $C$ to $Z_i$ (Figure \ref{fig:causalgraph} (c)). We denote the intervention procedure as function $I$, and rewrite the generative interventions as:
 \[\X_{int} = I(t, s, k, \Y')\]
 Based on our Theorem~\ref{th:2IV}, we choose the hyper-parameters $t, k, s$ for intervention $Z$ that maximizes $P(\x|\z)$. We show ablation studies in Section \ref{sec:int_effect}.


\begin{table*}[t]
\begin{center}
      \vspace{-3mm}
    \small
    \centering
    \begin{tabular}{l|cc|cc|cc|cc}
         \toprule
         & \multicolumn{4}{c|}{ResNet 18}  & \multicolumn{4}{c}{ResNet 152} \\
         & \multicolumn{2}{c|}{Std.\ Augmentation}  & \multicolumn{2}{c|}{Add.\ Augmentation} & \multicolumn{2}{c|}{Std.\ Augmentation}  & \multicolumn{2}{c}{Add.\ Augmentation}\\
          Training Distribution & top1   & top5  & top1 & top5 & top1   & top5  & top1 & top5\\
         \midrule  
          ImageNet Only \cite{He_2016, Objectnet}  & 20.48\% & 40.64\% & 24.42\% & 44.39\% & 30.00\% & 48.00\% & 37.43\% & 59.10\% \\
         Stylized ImageNet \cite{imagenetbiased} & 18.39\% & 37.29\% & 22.81\% & 42.27\% & 31.64\% & 52.56\% & 36.17\% & 57.95\%  \\
         Mixup \cite{zhang2017mixup} & 19.12\% & 37.78\% & 24.05\% & 44.17\% & 34.27\% & 55.68\% & 38.61\% & 60.36\%\\
         AutoAug \cite{cubuk2018autoaugment} & 21.20\% & 41.26\% & 21.20\% & 41.26\% & 33.96\% & 55.81\% & 33.96\% & 55.81\% \\
          GAN Augmentation \cite{CAS}  & 20.63\%  & 39.77\% &  23.72\% & 43.67\% & 33.17\% & 54.59\% & 36.37\% & 58.88\% \\
          \midrule
          GenInt (ours)  & 22.07\% & \textbf{41.94\%} &  25.71\% & 46.39\% & 34.47\% & 55.63\% & 39.21\% & 61.06\%\\
          GenInt with Transfer (ours) & \textbf{22.34\%} & 41.65\% & \textbf{27.03\%} & \textbf{48.02\%} & \textbf{34.69\%} & \textbf{55.82\%} & \textbf{39.38\%} & \textbf{61.43\%}\\
        
         \bottomrule
    \end{tabular}
\end{center}
      \vspace{-2mm}
\caption{Accuracy on the ObjectNet test set versus training distributions. By intervening on the training distribution with generative models, we obtain the state-of-the-art performance on the ObjectNet test set, even though the model was never trained on ObjectNet.} \label{tab:objectnet}
\end{table*}

\subsection{Transfer to Natural Data}\label{sec:transfer}

Maintaining the original training data $\X$ will add confounding bias to models. While our theory shows that our method still tightens the causal effect bound under the presence of spurious correlations, it is desirable to eliminate as many spurious correlations as possible. We will therefore also intervene on the original dataset.


One straightforward approach is to estimate the latent codes in the generator corresponding to the natural images, and apply our above intervention method.  We originally tried projecting the images back to the latent space in the generative models  \cite{zhu2018generative, huh2020ganprojection}, but this did not obtain strong results, because the projected latent code cannot fully recover the query images \cite{bau2019seeing}.

Instead, we propose to transfer the desirable generative interventions from $\X_{int}$ to the original data $\X$ with neural style transfer \cite{style_transfer}. The category information is maintained by the matching loss while the intervened nuisance factors are transferred via minimizing the maximum mean discrepancy \cite{li2017demystifying}. Without projecting the images to the latent code, the transfer enables us to intervene on some of the nuisance factors $z$ in the original data, such as the background. The transfer of the generative interventions $I(t, k, s, \Y')$ to natural data $\X$ is formulated as: 
\[\X_{itr} = T(I(t, k, s, \Y'), \X)\]
where $T$ denote the style transfer mapping. The corresponding label $\Y''$ is the same label as for $\X$. Please see supplemental material for visualizations of these interventions.

%% file: source/experiment.tex
\begin{table*}[t]
\centering
\setlength{\tabcolsep}{3.5pt}
\scriptsize
\vspace{-3mm}
    \begin{tabular}{ll|c|ccc|cccc|cccc|cccc}
         \toprule
         & Model & \textbf{mCE $\big\downarrow$} & Gauss. & Shot & Impulse & Defocus & Glass & Motion & Zoom & Snow & Frost & Fog & Bright & Contrast & Elastic & Pixel & JPEG \\
         \midrule  
          & AlexNet & 100.00 &100 &100 &100 &100 &100 &100 &100 &100 &100 &100 & 100 &100 &100 &100 &100  \\
          \midrule
          \parbox[t]{3mm}{\multirow{6}{*}{\rotatebox[origin=c]{90}{ResNet 18 \cite{He_2016}}}} & ImgNet Only \cite{He_2016} & 87.16 & 89.5&	90.4&	93.0 &	86.01 &	93.3 &	87.7 &	90.0 &	87.5 &	86.4 &	80.0 &	73.7 &	80.5 &	91.5 &	85.5 &	92.4 \\    
          & Stylized ImgNet \cite{imagenetbiased} & 80.83 & 79.1&	80.9&	81.7&	81.7&	87.6&	80.0&	90.0&	78.3&	80.2&	76.2&	72.5&	77.2&	\textbf{84.1} &	76.2&	86.7\\
          & Mixup \cite{zhang2017mixup} & 86.06 & 86.8 &	88.1&	90.8&	88.7&	95.6&	89.1&	89.3&	82.5&	\textbf{72.8}&	71.9&	75.9&	76.5 &	96.2&	89.5&	97.2\\
          & AutoAug \cite{cubuk2018autoaugment}  & 84.00 & 84.3&	83.7&	84.5&	87.9&	93.6&	87.7 &	93.5&	85.7&	83.4&	\textbf{71.0}&	\textbf{67.4}&	\textbf{63.5}&	97.8&	85.3&	90.5 \\
          & GAN Augmentation \cite{CAS} & 86.48 & 86.4&	87.5&	90.5&	87.0&	92.4&	87.2&	90.3&	88.0&	86.3&	82.8&	73.3&	82.8&	90.7&	84.5&	87.5 \\
           & GenInt (ours) & \textbf{74.68} & \textbf{67.0}&	\textbf{68.4}&	\textbf{67.3}&	\textbf{75.0}&	\textbf{80.5}&	\textbf{76.0}&	\textbf{84.2}&	\textbf{77.4}&	75.9&	77.5&	68.8&	76.6 &	87.5&	\textbf{59.8}&	\textbf{77.5}  \\ 
           \midrule
        \parbox[t]{3mm}{\multirow{6}{*}{\rotatebox[origin=c]{90}{ResNet 152 \cite{He_2016}}}}    & ImgNet Only \cite{He_2016} & 69.27 & 72.5&	73.4&	76.3&	66.9&	81.4&	65.7&	74.5&	70.7&	67.8&	62.1&	51.0&	67.1&	75.6&	68.9&	65.1 \\
           & Stylized ImgNet \cite{imagenetbiased} & 64.19& 63.3&	63.1&	64.6&	66.1&	77.0&	63.5&	71.6&	\textbf{62.4} &	65.4&	59.4&	52.0&	62.0&	73.2&	55.3&	62.9 \\
           & Mixup \cite{zhang2017mixup} & 66.43& 69.0& 	71.1& 	73.8& 	67.3& 	83.4& 	65.5& 74.6& 	63.5& 	\textbf{56.9}& 	\textbf{55.2} & 	49.4& 	62.4& 	75.4& 	65.0& 	63.7 \\
           & AutoAug \cite{cubuk2018autoaugment}  & 69.20 & 71.7&	72.8&	75.6&	67.2&	82.1&	67.7&	76.7&	70.3&	67.7&	61.8&	50.5&	65.0&	76.0&	68.3&	64.6 \\
           
           & GAN Augmentation \cite{CAS} & 69.01 & 71.8&	73.1&	75.9&	67.3&	82.3&	67.5&	76.2&	69.9&	68.1&	59.2&	51.3&	62.5&	76.6&	67.7&	65.7 \\
           
           & GenInt (ours) & \textbf{61.70} & \textbf{59.2}&	\textbf{60.2}&	\textbf{62.4}&	\textbf{60.7}&	\textbf{70.8}&	\textbf{59.5}&	\textbf{69.9}&	64.4&	63.8&	58.3&	\textbf{48.7}&	\textbf{61.5}&	\textbf{70.9}&	\textbf{55.2}&	\textbf{60.0} \\
           
         \bottomrule
    \end{tabular}
\vspace{3px}
\caption{The mCE $\downarrow$ rate (the smaller the better) on ImageNet-C validation \cite{imgnet-C} set with 15 different corruptions. Our GenInt model, without training on any of the corruptions, reduces the mCE by up to \textbf{12.48\%}. From column `Gauss.' to column `JPEG,' we show individual Error Rate on each corruption method. Without  adding similar corruptions in the training set, our generative causal learning approach learns models that naturally generalize to unseen corruptions.}
\label{tab:imagenet-C}
\end{table*}

\begin{table*}[t]
    \centering
    \vspace{-2mm}
    \footnotesize
    \begin{tabular}{ll|cc|cc|cc||cc}
         \toprule
         & & \multicolumn{6}{c||}{ImageNet-V2 Grouped by Sampling Strategy \cite{ImageNetOverfit}} & \multicolumn{2}{c}{Original}  \\
         & & \multicolumn{2}{c|}{``TopImages''} & \multicolumn{2}{c|}{``Threshold0.7''} & \multicolumn{2}{c||}{``MatchedFrequency''} & \multicolumn{2}{c}{ImageNet Val}  \\
         & Training Distribution & top1   & top5 & top1   & top5 & top1   & top5  & top1   & top5 \\
         \midrule  
          \parbox[t]{3mm}{\multirow{7}{*}{\rotatebox[origin=c]{90}{ResNet 18 \cite{He_2016}}}} & ImageNet Only \cite{He_2016} & 71.77\% & 91.11\% & 65.41\% & 87.39\% & 56.18\% & 79.35\% & 68.82\% & 88.96\%\\    
          & Stylized ImageNet \cite{imagenetbiased} & 69.55\% & 89.97\% & 62.92\% & 85.38\% & 54.13\% & 77.30\% & 66.95\% & 87.42\%\\
          & Mixup \cite{zhang2017mixup} & 69.90\% & 90.16\% & 63.42\% & 86.40\% & 54.42\% & 77.94\% & 66.00\% & 86.93\%\\
          & AutoAug \cite{cubuk2018autoaugment} & 72.05\% & 91.49\% & 65.32\% & 87.32\% & 56.25\% & 79.16\% & 69.24\% & 88.91\%\\
          
           & GAN Augmentation \cite{CAS}  & 72.01\% & 91.24\% & 65.72\% & 87.58\% & 56.43\% & 79.42\% & 69.19\% & 88.85\% \\
           \cmidrule{2-10}
           & GenInt (ours) & 72.80\% & \textbf{91.89\%} & 66.26\% & \textbf{88.30\%} & \textbf{57.86\%} & \textbf{80.11\%} & \textbf{70.41\%} & \textbf{89.59\%} \\
           
           & GenInt with Transfer (ours) & \textbf{72.84\%} & 91.85\% & \textbf{66.49\%} & 88.11\% & 57.35\% & 79.61\% & 70.25\% & 89.33\%\\
           
           \midrule
           \parbox[t]{3mm}{\multirow{7}{*}{\rotatebox[origin=c]{90}{ResNet 152 \cite{He_2016}}}}   & ImageNet Only \cite{He_2016} & 81.01\% & 96.21\% & 76.17\% & 94.12\% & 67.76\% & 87.57\% & 78.57\% & 94.29\% \\    
          & Stylized ImageNet \cite{imagenetbiased} & 79.40\% & 95.72\% & 74.02\% & 92.88\% & 65.12\% & 86.22\% & 77.27\% & 93.76\% \\
          & Mixup \cite{zhang2017mixup} & 80.68\% & 96.28\% & 75.91\% & 94.00\% & 67.11\% & 87.66\% & 78.78\% & 94.45\% \\
          & AutoAug \cite{cubuk2018autoaugment} & 80.61\% & 96.30\% & 75.90\% & 94.06\% & 67.35\% & 87.61\% & 78.95\% & 94.56\%\\
          
           & GAN Augmentation \cite{CAS}  & 80.10\% & 96.00\% & 75.60\% & 93.74\% & 66.89\% & 87.04\% & 78.53\% & 94.21\% \\
           \cmidrule{2-10}
           & GenInt (ours) & 80.77\% & \textbf{96.38\%} & 76.20\% & \textbf{94.24\%} & 67.74\% & \textbf{87.83\%} & 79.46\% & 94.71\% \\
           
           & GenInt with Transfer (ours) & 
            \textbf{81.24\%} & 96.28\% & \textbf{76.60\%} & 93.95\% & \textbf{68.08\%} & 87.70\%& \textbf{79.59\%} & \textbf{94.79\%} \\
           
         \bottomrule
    \end{tabular}
\vspace{3px}
\caption{Accuracy on ImageNet V2 validation set \cite{ImageNetOverfit} and original ImageNet validation set. Our method improves the performance upon the baselines, which suggests our causal learning approach does not hurt the performance on original test set while becoming robust.}
\vspace{-5mm}
\label{tab:imagenet-v2-18}
\end{table*}

\section{Experiments}


We present image classification experiments on four datasets --- ImageNet, ImageNet-V2, Imagenet-C, and ObjectNet --- to analyze the generalization capabilities of this method and validate our theoretical results. We call our approach \textbf{GenInt} for generative interventions, and compare the different intervention strategies. 



\subsection{Datasets}

In our experiments, all the models are first trained on \textbf{ImageNet} \cite{imagenet_cvpr09}  (in addition to various intervention strategies). We train only on ImageNet without any additional data from other target domains. We directly evaluate the models on the following out-of-distribution testing sets:

\textbf{ObjectNet} \cite{Objectnet} is a test set of natural images that removes background, context, and camera viewpoints confounding bias. Improving performance on ObjectNet---without fine-tuning on it---indicates that a model is learning causal features. ObjectNet's policy prohibits any form of training on the ObjectNet data. We measure performance on the 113 overlapping categories between ImageNet and ObjectNet. 
 
\textbf{ImageNet-C} \cite{imgnet-C} is a benchmark for model generalization under 15 common corruptions, such as 'motion,' 'snow,' and 'defocus.' Each corruption has 5 different intensities. We use mean Corruption Error (mCE) normalized by AlexNet as the evaluation metric \cite{imgnet-C}. Note that we do not train our model with any of these corruptions, thus the performance gain measures our model's generalization to unseen corruptions.
 
\textbf{ImageNet-V2} \cite{ImageNetOverfit} is a new test set for ImageNet, aiming to quantify the generalization ability of ImageNet models. It contains three sampling strategies: MatchedFrequency, Threshold0.7, and TopImages. While current models are overfitting to the ImageNet test set, this dataset measures the ability to generalize to a new test set.




\subsection{Baselines}
We compare against several established data augmentation baselines:

\textbf{Stylized ImageNet} refers to training the model using style transferred dataset \cite{imagenetbiased}, which trains classifiers that are not biased towards texture.

\textbf{Mixup} \cite{zhang2017mixup} does linear interpolation to augment the dataset. We use their best hyperparameters setup ($\alpha=0.4$).

\textbf{AutoAug} \cite{cubuk2018autoaugment} systematically optimizes the strategy for data augmentation using reinforcement learning.

\textbf{GAN Augmentation} refers to the method that augments the ImageNet data by directly sampling from the BigGAN \cite{CAS}. They provide  an extensive study for hyper-parameter selection. We use their best setup as our baseline: 50\% of synthetic data sampled from BigGAN with truncation 0.2. 

\textbf{ImageNet only} refers to training the standard model on ImageNet dataset only \cite{He_2016}.

\subsection{Empirical Results}

 Our GenInt method demonstrates significant gains on four datasets over five established baselines.
 We report results for two different network architectures (ResNet18, ResNet152). All ResNet18 models are trained with SGD for 90 epochs, we follow the standard learning rate schedule where we start from 0.1, and reduce it by 10 times every 30 epochs.
For ResNet152 models, we train ``ImageNet only'' models using the above mentioned method, and finetune all the other methods from the baseline for 40 epochs given that it is computationally expensive to train ResNet-152 models from scratch. 
For GenInt, we all use $\lambda_1=0.05$ and $\lambda_2=0$ for ResNet18 and $\lambda_1=0.2$ and $\lambda_2=0$ for ResNet152. For \textbf{GenInt with Transfer}, we use $\lambda_1=0.02$ and $\lambda_2=1$ for our experiments on Resnet18 with standard augmentation, $\lambda_1=0.05$ and $\lambda_2=1$ for our experiments on Resnet18 with additional augmentation, and $\lambda_1=0.2$ and $\lambda_2=0.2$ for our finetuning on ResNet152. We select hyperparameters of our intervention strategy in Section \ref{sec:int_effect}. Implementation details are in the supplementary.

\textbf{ObjectNet:}
Table \ref{tab:objectnet} shows that our model can learn more robust features, and consequently generalizes better to ObjectNet without any additional training. 
Our results consistently outperform the naive sampling from generative models \cite{CAS} and other data augmentation strategies \cite{zhang2017mixup, imagenetbiased, cubuk2018autoaugment} for multiple metrics and network architectures, highlighting the difference between traditional data augmentation and our generative intervention. Our approach enjoys benefits by combining with additional data augmentations, demonstrated by the differences between the ``Std.\ Augmentation'' columns and the ``Add.\ Augmentation'' columns.\footnote{Standard augmentation only uses random crop and horizontal flips \cite{pytorch_imagenet}. Additional augmentation method uses rotation and color jittering \cite{tian2020rethinking}.} This improvement suggests that our generative intervention can manipulate additional nuisances (viewpoints, backgrounds, and scene contexts) orthogonal to traditional augmentation, which complements existing data augmentation methods. Moreover, our results suggest that intervening on the generative model is more important than just sampling from it.

\begin{figure}[t]
    \centering
    \includegraphics[width=1.0\linewidth]{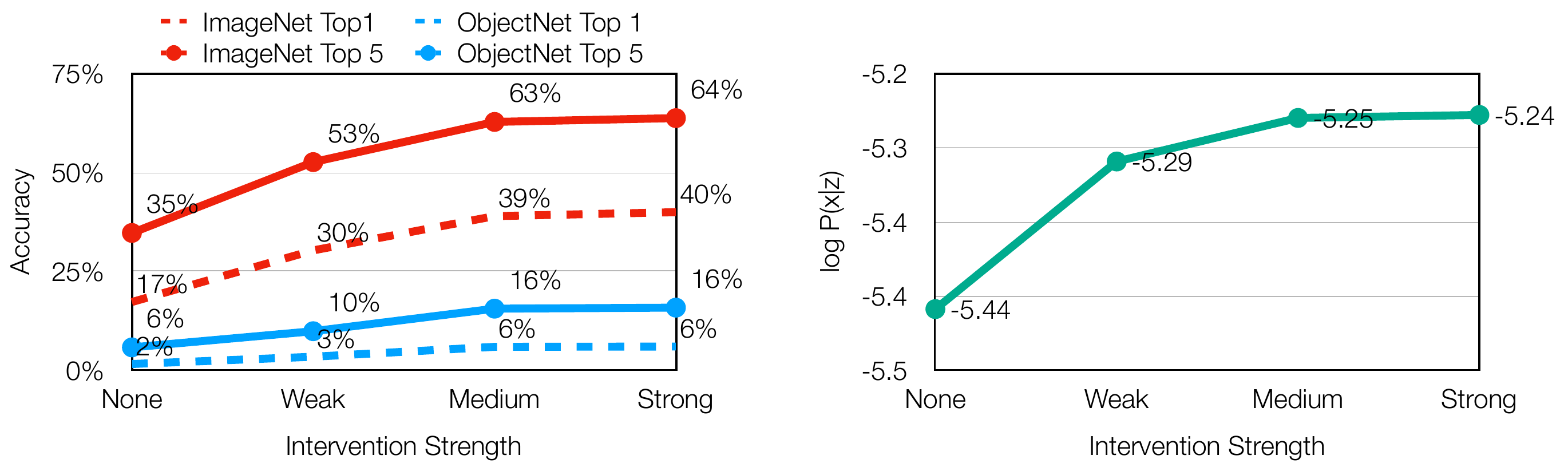}
    \caption{As the strength of the intervention increases, the value of $\log P(x|z)$ increases, which improves the performance of ResNet-18 model. }
    \label{fig:intervention}
    \vspace{-4mm}
\end{figure}

\begin{figure}[t]
    \centering
    \includegraphics[width=\linewidth]{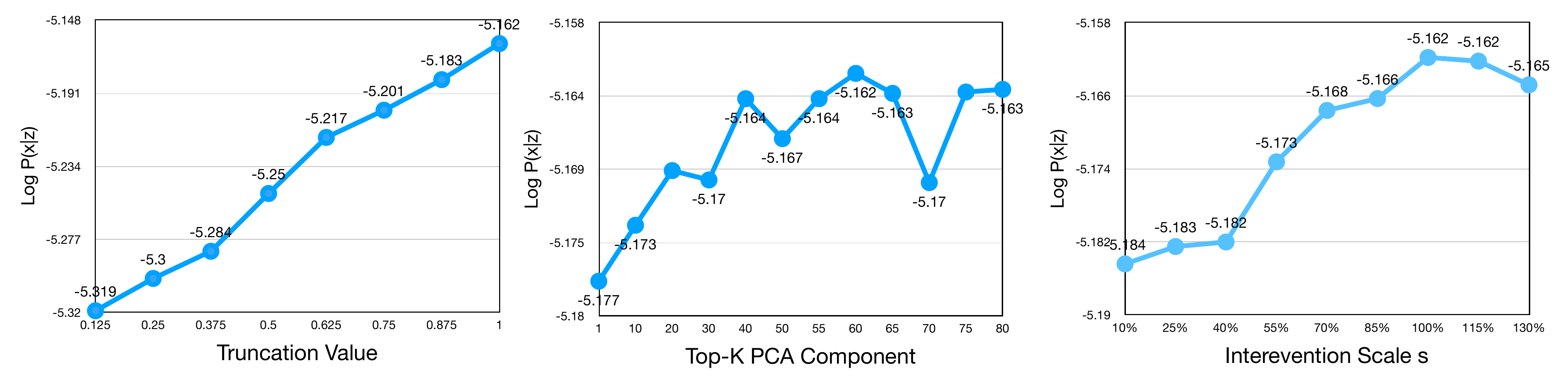}
    \caption{$\log P(x|z)$ for causal effect bound under different intervention strategies. The x-axis of each subfigure changes one hyper-parameter for intervention strategy: truncation value $t$ (left), PCA number $k$ (middle), and the intervention scale $s$ (right). Based on theorem \ref{th:2IV}, we choose the hyper-parameters $t,k,s$ that produces the highest value for $\log P(x|z)$ from individual figure.}
    \label{fig:intervention_ablation}
    \vspace{-4mm}
\end{figure}

\textbf{ImageNet-C:} To further validate that our approach learns causality, and not just overfits, we measure the same models' generalization to unseen corruptions on ImageNet-C. We evaluate performance with mean corruption error (mCE) \cite{imgnet-C} normalized by CE of AlexNet. Table \ref{tab:imagenet-C} shows that directly sampling from GAN as augmentation (GAN Augmentation) slightly improves performance (less than 1\%). Stylized ImageNet achieves the best performance among all the baselines, but it is still worse than our approach in mCE. In addition, Stylized ImageNet hurts the performance on ObjectNet, which suggests its high performance on corruptions is overfitting to the correlations instead of learning the causality. Our approach outperforms baseline by up to 12.48\% and 7.57\% on ResNet18 and ResNet152 respectively, which validates that our generative interventions promote causal learning.


\textbf{ImageNet and ImageNet-V2:} Table \ref{tab:imagenet-v2-18} shows the accuracy on both validation sets. Some baselines, such as Stylized ImageNet, hurt the performances on the ImageNet validation set, while our approach improves the performance.  

Overall, without trading-off the performance between different datasets, our approach achieves improved performance for all test sets, which highlights the advantage of our causal learning approach.




\begin{table}
\centering
    \centering
    \scriptsize
    \begin{tabular}{lc|cc}
         \toprule
         & Truncation & \multicolumn{2}{c}{ImageNet}  \\
          Training Dist.& & top1&top5\\
         \midrule  
         Obervational GAN \cite{CAS}  & 1.0 & 39.07\% & 62.97\%  \\
         Obervational GAN \cite{CAS} & 1.5 & 42.65\% & 65.92\% \\
         Obervational GAN \cite{CAS} & 2.0 & 40.98\% & 64.37\%\\
         Interventional GAN (ours) & 1.0 & \textbf{45.06\%} & \textbf{68.48\%}\\
         \bottomrule
    \end{tabular}
\vspace{3px}
\caption{We show performance for ResNet50 trained only on BigGAN. Our intervention model surpasses performance of the best established benchmark \cite{CAS}} \label{tab:CAS}
    \vspace{-5mm}
\end{table}

\subsection{Analysis}\label{sec:int_effect}

\textbf{Causal Bound and Performance:} Does tighter causal bound lead to a better classifier? Following Theorem \ref{th:2IV}, we measure the tightness of causal bound after intervention, where we use the log likelihood $\log P(x|z)= \sum_{i}\sum_{x'_j} \log (P(x_i|x'_j)P(x'_j|z))$, where $x_i$ is the query image from the held out ImageNet validation set, and $x'_j$ is the data generated by intervention $z$. We train ResNet18 on our generated data.\footnote{We sample an observational and intervention data from BigGAN with truncation $t=0.5$ \cite{BigGAN}.  Please see supplementary material for full details. }
By varying the intervention strength, we increase the value of $P(x|z)$, which corresponds to a tighter causal bound. Figure \ref{fig:intervention} shows that, as the causal bound getting tighter (left), performance steadily increases (right). 


\textbf{Optimal Intervention Strategy:} Since tighter causal bound produces better models, we investigate the optimal intervention strategy for tightening causal bounds. We study the effect of changing $t, k, s$ for our intervention on the causal bound (Section \ref{method:genint}). We conduct ablation studies and show the trend in Figure \ref{fig:intervention_ablation}. We choose $t=1$, $k=60$, and $s=100\%$ as our intervention strategy for tightest causal bound, which produces $\log P(x|z)=-5.162$ and yields the optimal accuracy of 45.06\% (Table \ref{tab:CAS}) in practice.

\textbf{Importance of Intervention:}
Our results show that creating interventions with a GAN is different from simply augmenting datasets with samples from a GAN. To examine this,
Table \ref{tab:CAS} shows performance on ImageNet when the training sets only consist of images from the GAN. We use the best practices from \cite{CAS}, which comprehensively studies GAN image generation as training data. Our results show that creating interventions, not just augmentations, improves classification performance by 2.4\%-6.0\%.

\begin{figure}[t]
  \centering
      \vspace{-5mm}
  \includegraphics[width=0.47\textwidth]{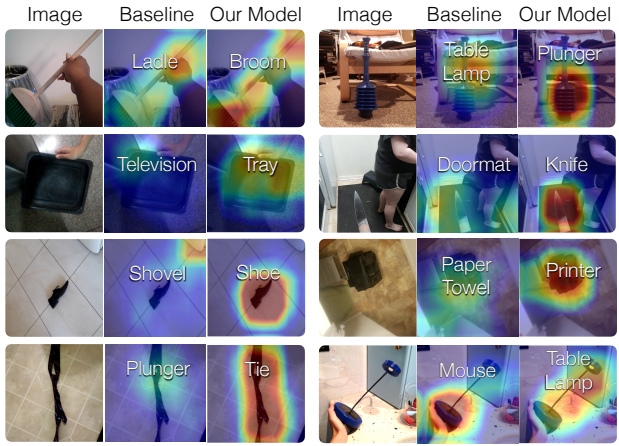}
  \caption{We visualize the input regions that the model uses to make predictions. Blue implies the model ignores the region for discrimination, while red implies the region is very discriminative. The white text shows the model's top prediction. The baseline frequently latches onto spurious background context (e.g., hand spuriously correlated with ladle, chair spuriously correlated with tablelamp), and consequently makes the wrong prediction. Meanwhile, our model often predicts correctly for the right reasons.}
  \label{fig:gradcam}
      \vspace{-5mm}
\end{figure}
\subsection{Model Visualization}


By removing the confounding factors in the dataset, we expect the model to learn to attend tightly to the spatial regions corresponding to the object, and not spuriously correlated contextual regions. To analyze this, Figure \ref{fig:gradcam} uses GradCAM \cite{gradcam} to visualize what regions the models use for making prediction. While
the baseline often attends to the background or other nuisances for prediction, our method focuses on the spatial features of the object. For example, for the first `Broom' image, the baseline uses spurious context `hand,' leading to a misprediction `Ladle,' while our model predicts the right `Broom' by looking at its shape. This suggests that, in addition to performance gains, our model predicts correctly for the right reasons. 





%% file: source/conclusion.tex
\section{Conclusion}

Fortifying visual recognition for an unconstrained environment remains an open challenge in the field.
We introduce a method for learning discriminative visual models that are consistent with causal structures, which enables robust generalization. By steering generative models to construct interventions, we are able to randomize many features without being affected by confounding factors. We show a theoretical guarantee for  learning causal classifiers under imperfect interventions, and demonstrate improved performance on ImageNet, ImageNet-C, ImageNet-V2, and the systematically controlled ObjectNet.

%% file: supp/supp_main.tex
\newpage
\newpage
\begin{subappendices}




%
\input{supp/source/ColorfulMNIST}
\input{supp/source/theory.tex}
\input{supp/source/causalbound}
\input{supp/source/CorrelationAnalysis.tex}

\input{supp/source/experiment.tex}

\input{supp/source/visualizations.tex}

\clearpage

\bibliography{reference}
\bibliographystyle{plain}



\end{subappendices}

%% file: supp/source/ColorfulMNIST.tex
\section{Motivating Experiment}

We create a controlled experiment to show that generative models can construct interventions for causal learning. Our experiments demonstrate that generative models can inherently discover nuisances and intervene on the confounding factors, creating extrapolated data beyond the training distribution with confounding bias.  

\begin{table}
\centering
    \centering
    \small
    \begin{tabular}{l|c|c}
         \toprule
         & Confounded & Causal \\
         & Test Accuracy  &  Test Accuracy\\
         \midrule  
         Chance & 10\% & 10\%  \\
         Original Data  & \textbf{99.45\%} & 8.261\%  \\
         IRM \cite{arjovsky2019invariant} & 87.32\% & 18.49\% \\
         Observational CVAE \cite{CVAE} & 59.949\%& 11.255\%\\
         Interventional CVAE & 58.478\%& \textbf{29.618\%}  \\
         \bottomrule
    \end{tabular}
\caption{We show 10-way classification accuracy on the Colored MNIST dataset. Color is a spurious correlation that no longer holds during the causal test. Our generative intervention strategy advances the state-of-the-art IRM method by 11.12\%} \label{tab:colorfulmnist}
\end{table}

\begin{figure}
  \centering
  \subfloat[][Background 1]{\includegraphics[width=0.25\linewidth]{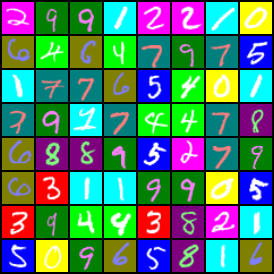}\label{scm1}}
  \hspace{5em}
  \subfloat[][Background 2]{\includegraphics[width=0.25\linewidth]{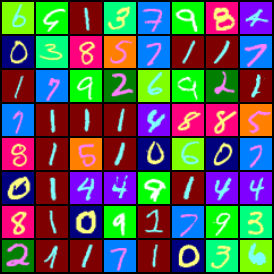}\label{scm2}}
  \caption{Illustration of color MNIST dataset. For each digit category, we generate two different background colors. The feature background color is spuriously correlated to the category, where the confounder is us, the dataset creator. But the observed data is only color digits and corresponding targets.
  }
  \label{fig:train-color}
\end{figure}

\begin{figure*}
  \centering
  \includegraphics[width=1.0\textwidth]{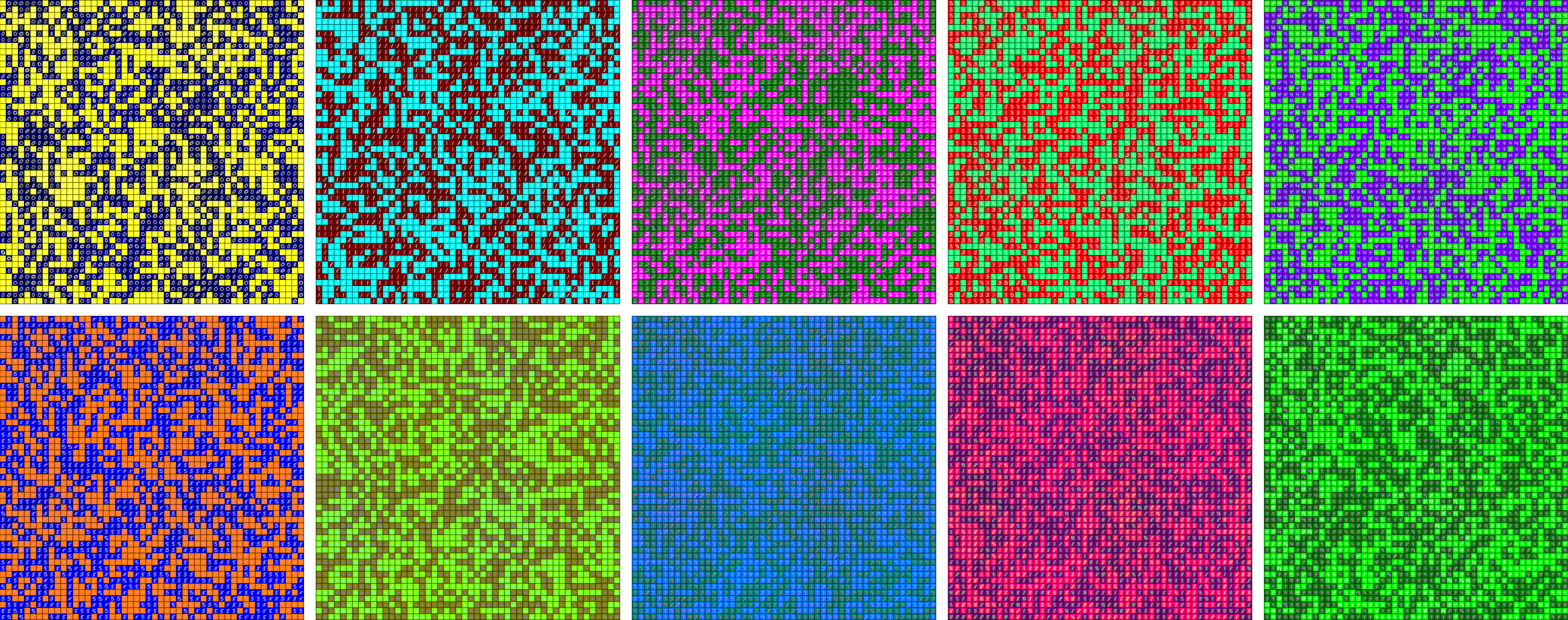}
  \vspace{-5mm}
  \caption{The background color for each class without intervention in the observational VAE model. The generator \emph{fails} to generate digits with different background color from the training set, which demonstrates the importance of intervention.}
  \label{fig:vae-grid-ob}
\end{figure*}

\begin{figure*}
  \centering
  \includegraphics[width=1.0\textwidth]{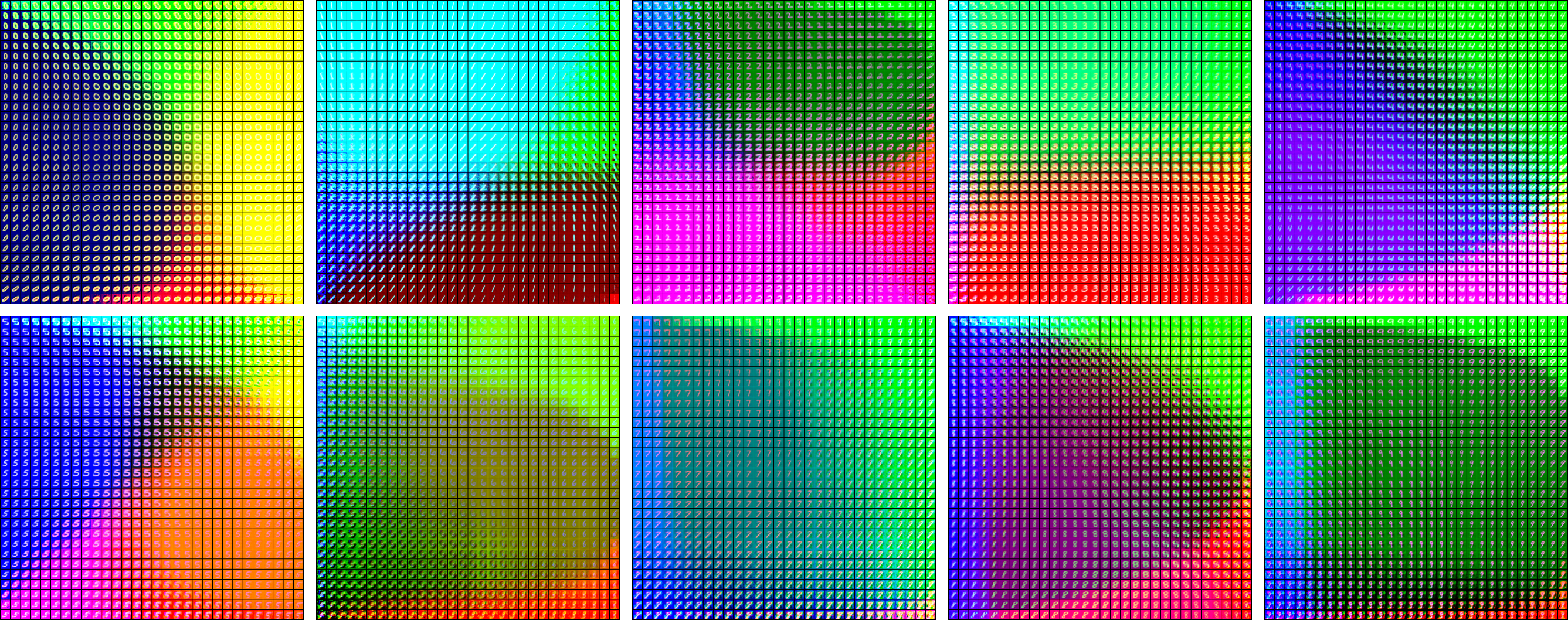}
  \vspace{-5mm}
  \caption{The background color for each class in interventional CVAE. We intervene on two principal component directions in the latent space. Despite the dataset being created with only 2 colors per category, new background colors emerge in the generative model after interventions. This demonstrates the importance of intervening on generative models for creating unbiased data.}

  \label{fig:vae-grid}
  \vspace{-5mm}
\end{figure*}

\begin{figure}
  \centering
  \includegraphics[width=0.4\textwidth]{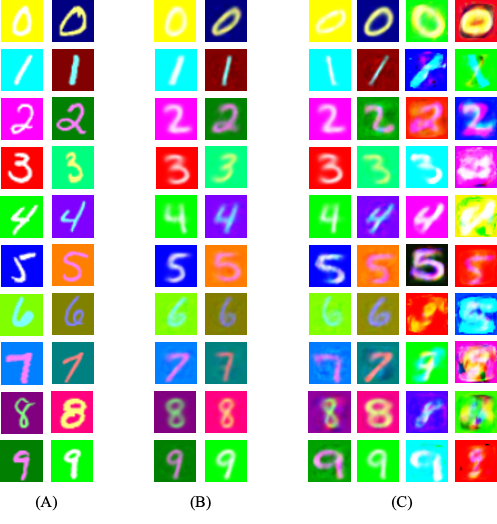}
  \caption{Comparison of the background color for each class for the original dataset (A), the observational CVAE (B), and the interventional CVAE (C). New colors emerge after intervening on the generative model. The background color is randomized after interventions, so that it no long spuriously correlates with the target label.}
  \label{fig:compare-mnist}
\end{figure}

\subsection{Controlled Example: Colored MNIST}

\textbf{Dataset:} We analyze our model on MNIST, which allows us to carefully control experiments. We use \emph{Colored MNIST}~\cite{arjovsky2019invariant} where we explicitly introduce a confounding bias $c$ that affects the background color $b$ for the training set. This confounding bias does not exist in the test set. We set two different background colors for each digit category $y_i$, $i = 0 \dots 9$. While the handwritten digit is the true cause, the background color is spuriously correlated to the target. A classifier that makes predictions based on the spurious background correlation will have $P(y_i|b) \neq \frac{1}{10}$, while the causal classifier learns to be invariant to background-color $P(y_1|do(b)) = P(y_2|do(b)) = \frac{1}{10} $. We show training examples for our manipulated colored MNIST in Figure \ref{fig:train-color}. By training models under Empirical Risk Minimization, the model will learn the spurious correlations instead of the true causal relationship, thus not generalizing well once the spurious correlations change.

\textbf{Experimental Setup:} We validate this outcome by experiment. The baseline is trained only on the original colored data. For methods involving a generator, we train a conditional variational auto-encoder (CVAE) \cite{CVAE} on the observed joint distribution. Observational CVAE denotes the classifier only trained on data sampled from an original CVAE without intervention, corresponding to the `GAN Augmentation' method in our main paper. Our proposed method is labeled Interventional CVAE, corresponding to `GenInt' in our main paper, where we train classifiers on data generated by intervening on the two principal directions in latent space. The intervention scale is uniformly sampled from a given interval to cut off the dependency on the category. Since the generative model captures the MNIST joint distribution well, we use only the generative interventional data without the original data.

\subsection{Results}

\textbf{Visualization Results:} In Figure \ref{fig:vae-grid-ob}, we visualize data generated by observational CVAE. As we can see, the color produced by observational CVAE is the same as the training set---no new color emerges---which is due to observational CVAE capturing the same joint distribution as the training set.
In Figure \ref{fig:vae-grid}, we visualize what happens to the generated image background once we intervene on two major directions in the generative models. New colors emerge due to our intervention, thus we can randomize the features affected by the confounding bias through proper intervention. Figure \ref{fig:compare-mnist} also demonstrates the difference between the original dataset, the observational VAE, and the interventional VAE.

\textbf{Invariant Risk Minimization:} Recent work \cite{arjovsky2019invariant} proposes an Invariant Risk Minimization(IRM) algorithm to estimate invariant correlations across multiple training distributions, which is related to causal learning and enables out-of-distribution generalization. IRM algorithm is shown to work on binary colored MNIST classification. We implement IRM on our more challenging 10-way colored MNIST dataset. Table \ref{tab:colorfulmnist} shows that our algorithm bypasses it by a big margin on the ``Causal Test Accuracy,'' which demonstrates the effectiveness of our algorithm for causal learning in confounded data.

\textbf{Quantative Results: } Table \ref{tab:colorfulmnist} shows that the baseline model performs well when the confounding bias persists in the test set, but catastrophically fails (worse than chance) once the spurious colors are changed. The CVAE suffers from the same issue, demonstrating that data augmentation with a CVAE is insufficient to learn robust models. The state-of-the-art solution on colorful MNIST is the IRM, where the classifier is optimized across several different environments. Our method, Interventional CVAE, doubles the accuracy on the causal test set without substantial decreases on the confounded test set, we also advance the state-of-the-art IRM method by more than 10\%. The results show that using generative interventions, our approach can learn causal representations more effectively than non-interventional methods.


%% file: supp/source/theory.tex
\begin{figure*}[t]
  \centering
  \includegraphics[width=0.8\textwidth]{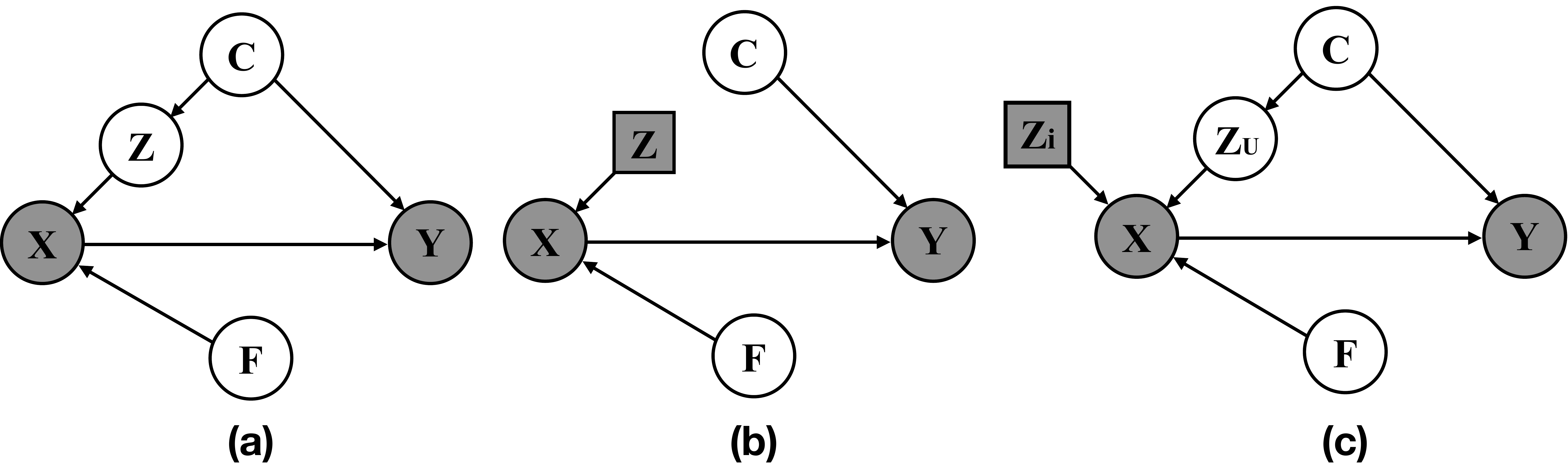}
  \caption{Causal graph for image classification. Gray variables are observed. (\textbf{a}) $F$ is the unobserved variable that generates the object features. The unobserved confounder $C$ causes both the  background features $Z$ and label $Y$, which creates a spurious correlation between the image $X$ and label $Y$. (\textbf{b}) An ideal intervention blocks the backdoor path from $Z$ to $C$, which produces causal models. (\textbf{c}) In practice, we cannot guarantee to intervene on all the $Z$ variables. However, by properly intervening on even a small set of nuisance factors $Z_i$, the confounding bias of the observed distribution is mitigated, which is theoretically proven by our theorem.}

  \label{fig:causalgraph}
\end{figure*}

\section{Proof for Theoretical Analysis}

We formalize the framework with Structural Causal Models (SCMs) (Pearl, 2000, pp. 204-207). An SCM contains $<U,V,\phi,P(U)>$, where $U$ is a set of unobserved variables and $V$ is the observed variables, $\phi$ is the set of dependence functions, and $P(U)$ encodes the uncertainty of the exogenous variables. In our paper, $V=\{X,Y\}$, $U=\{C,F,Z, U_x, U_y\}$, we do not plot $ U_x, U_y$ on the causal graph explicitly. We assume that $ U_x, U_y$ are exogenous variables that capture the uncertainty of variables $X$ and $Y$, respectively.  $C$ is the unobserved confounding bias, which causes the object image $X$ and its corresponding label $Y$.  We validate the existence of $C$ via the correlations analysis experiment in Section 3.1. We assume $P(U)$ satisfies the Gaussian distribution, but it can also be any other distribution in our theory. We plot our causal graph in Figure \ref{fig:causalgraph}
.
The functional relationship $\phi$ between the variables are as follows:

\begin{align*}
    & Z := \phi_z (C) \\
    & X := \phi_x (U_x, F, Z) \\
    & Y := \phi_y (U_y, X, C)
\end{align*}


Following \cite{pearl}, we define the causal effect of variable $X$ on $Y$ as follows:

\begin{definition}
The causality of variable $X$ on $Y$, denoted as $P(Y|do(X))$, is the effect of conducting $X$ to $Y$ while keeping all the other variables the same.
\end{definition}

Note that $P(Y|do(X))$ is different from $P(Y|X)$. Since $P(Y|X)$ is the observational distribution, a change in $X$ can suggest a change in unobserved confounding bias under our causal graph. It is possible that the observed change in $X$ is due to changes in confounding variables.
The observed $P(Y|X)$ thus is not the same as $P(Y|do(X))$. To identify the causal effect of $P(Y|do(X))$, we need to observe the confounding bias, or intervene on $X$.

In Figure 1 in the paper, by observing variables that block the back-doors from $X$ to $Y$, we can learn the causal effect from $X$ to $Y$.

\begin{theorem} 
One can identify the causal effect of $X$ on outcome $Y$ by observing the hidden factors $C$ or $Z$.
\end{theorem}

\begin{proof}
We denote the parent node of variable $X$ in the causal DAG graph as $pa_x$.

Then for $C$:
\begin{align*}
    P(y|do(x)) & = \sum_{pa_x} P(y|x, pa_x)P(pa_x) \\
               & = \sum_{pa_x,c} P(y|x, pa_x, c)p(e,c|x,pa_x)P(pa_x) \\
               & = \sum_{pa_x,c} P(y|x, c) P(e,c|pa_x)P(pa_x) \\
               & = \sum_{c} P(y|x, c) P(c)
\end{align*}

Thus one can identify the causal effect of $X$ on $Y$ by observing $C$.

For $Z$:
\begin{align*}
    P(y|do(x)) & = \sum_{pa_x} P(y|x, pa_x)P(pa_x) \\
              & = \sum_{pa_x,z} P(y|x, pa_x, z)p(z|x,pa_x)P(pa_x) \\
              & = \sum_{pa_z} P(y|x, z) P(z|pa_x)P(pa_x) \\
              & = \sum_{z} P(y|x, z) P(z)
\end{align*}

Thus one can identify the causal effect of $X$ on $Y$ by observing $Z$.

\end{proof}

Given only the observational data, it is often impossible to identify the exact causal effect of data. Instead, \cite{pearl} proposed the natural bound for the causal effect to restrict the possible causal effect in a range.

\begin{lemma}
Given the observed joint distribution, the natural bound for $P(Y|do(x))$ is bounded by $[P(X,Y), P(X,Y)+1-P(X)]$.
\end{lemma}

\begin{proof}
We denote all the unobserved confounding factors as $u$.

\begin{align*}
    P(y|do(x)) & = \sum_u P(y|x,u)P(u) \\
    & = \sum_u P(y|x,u)(P(u,x)+P(u)-P(u,x)) \\
    & = \sum_u P(x,y,u) + \sum_u P(y|x,u)(P(u)-P(u,x)) 
\end{align*}

Since $0 \leq P(y|x,u) \leq 1$,

\begin{align}
    P(y|do(x)) & \geq \sum_u P(x,y,u) = P(x,y) \\
    P(y|do(x)) & \leq  \sum_u ( P(x,y,u) + (P(u)-P(u,x))) \\
    &= P(x,y) + (1 -P(x) )
\end{align}
This illustrates the proof for natural bound. 
\end{proof}

Ideally, we desire a tighter causal bound. As we can see in Figure \ref{fig:bound}, a tighter bound reduces the overlap of bound intervals, and creates a margin between the intervals of the probability of predicting different objects. If we can separate the interval of the maximum category from the others, we can predict the causal results even under bounded causal effect.

Although observing all the confounding factors is impossible for most vision tasks, there are some confounding factors that can be captured by generative models. We assume that we can intervene on a subset of the features caused by the confounders. We can tighten the causal bound using our intervention using generative models, which helps to learn models that are more consistent with the true causal effect. We propose the following theorem in the main paper:

\begin{theorem}\label{th:2IV}
We denote the images as $x$. The causal bound under intervention $z_i$ is thus $P(y,\x|z_i) \leq P(y|do(x)) \leq P(y,x|z_i) + 1 - P(x|z_i)$. For two intervention strategies $z_1$ and $z_2$,  $z_1 \subset z, z_2 \subset z$, if  $P(x|z_1) > P(x|z_2)$, then $z_1$ is more effective for causal identification. 

\end{theorem}\vspace{-1em}\begin{proof}
Figure 1b shows the causal graph after intervention where $Z= Z_i, Z_i \indep Y$.  We follow causal calculus rules by Pearl \cite{pearl} and add and remove the same term $\sum_c P(y,x,c|z_i)$:
\begin{align*}
P(y|do(x)) &= \sum_c P(y|x,z_i,c)P(c) \\
& = \sum_c P(y,x,c|z_i) + \\ 
& \sum_c{P(y|x,z_i,c)(P(c)-P(x,c|z_i))}
\end{align*}
Since $0 \leq P(y|x,z_i,c) \leq 1$, we have the lower and upper bounds.


We denote $\delta_1 = P(x|z_1) - P(x|z_2)$, thus  $\delta_1>0$. Since we intervene on $z_i$ in the causal graph of Figure \ref{scm2}, all incoming edges  to $z_i$ are removed. Therefore, $z_i \indep y | x$ and $P(x,y|z_i) = P(y|x,z_i)P(x|z_i)= P(y|x)P(x|z_i)$. Next, let $\delta_2=P(x,y|z_1)-P(x,y|z_2)$, then $\delta_2 = \delta_1 \cdot P(y|x)$. Since $0<P(y|x)<1$, then $0<\delta_2<\delta_1$. Thus we obtain $[P(y,x|z_1), P(y,x|z_1) + 1 - P(x|z_1)] \subset [P(y,x|z_2), P(y,x|z_2) + 1 - P(x|z_2)]$, which means the intervention $z_1$ results in a tighter causal effect bound.
\end{proof}

\begin{figure*}
  \centering
  \includegraphics[width=1.0\textwidth]{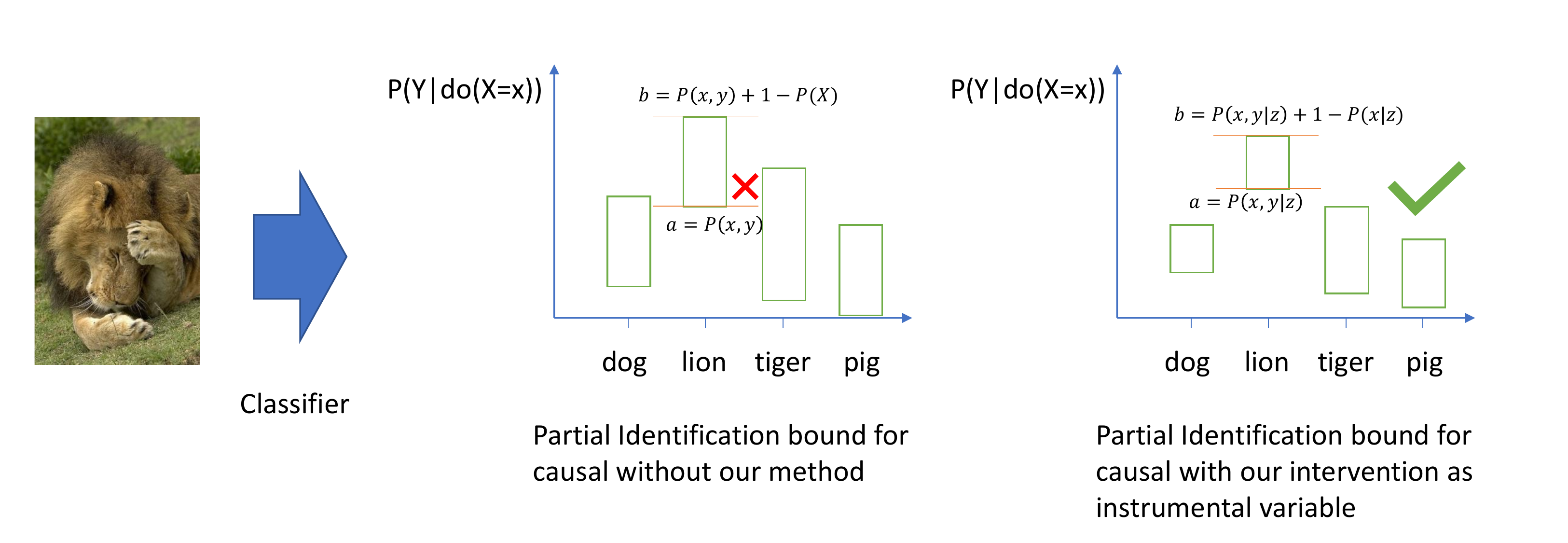}  
  \vspace{-5mm}
  \caption{Tightened identification bound reduces the overlap on the bound for probabilities of the predicted labels, which results in correct recognition when the bounds do not overlap. }
  \label{fig:bound}
\end{figure*}

Though deep learning methods are non-linear, to provide insights for the effect of intervening on a subset of variables, we analyze an example under a linear model. The following theorem \ref{th:linearIV} shows that, if the models are linear, then one can estimate the causal effect with intervened $Z$.

\begin{theorem}\label{th:linearIV}
If all the variables follow a Gaussian distribution, and the functional relationships between the variables are linear, then one can identify the causal effect by observing $Z$ in Figure 4c, even though there remain unobserved confounding factors.
\end{theorem}

\begin{proof}
Based on the linear assumption,
Let $Y=a_1 C + b X + u_y$,  $X=a_4 F + a_5 Z_i + a_6 Z_U + u_x$.
Our goal is to estimate $P(Y|do(X))$ which we denote as $b=P(Y|do(X))$. The causal graph is shown in Figure 2b in the main paper.

Under a linear model, we conduct linear regression to estimate the coefficient $\sigma_{zy}$ between $Z_i$ and $Y$.

\[\sigma_{zy} = a_5 b\]

The confounding factors do not appear in this regression, because X is an unobserved collider in the regression, thus information cannot go to the confounders.

Then we conduct linear regression to estimate the coefficient $\sigma_{zx}$ between $Z$ and $X$.

\[\sigma_{zx} = a_5\]

where the path from $z$ to $x$ is the causal path.

Thus we can estimate the causal effect $P(Y|do(X))=\frac{\sigma_{zy}}{\sigma_{zx}}=b$ under linear model.
\end{proof}

This linear example shows that by intervening only on a subset of the confounding factors $z$, one can identify the causal effect from $x$ to $y$ under unobserved confounders. Compared with Section 3.1.1 in \cite{ilse2020designing}, we show that even under a much weaker assumption where we can intervene on only a subset of nuisance factors, we can learn causal effect in linear models. This linear example also motivates the reason for effectiveness of our method in a non-linear setting.  

\begin{figure}[t]
  \centering
  \includegraphics[width=0.25\textwidth]{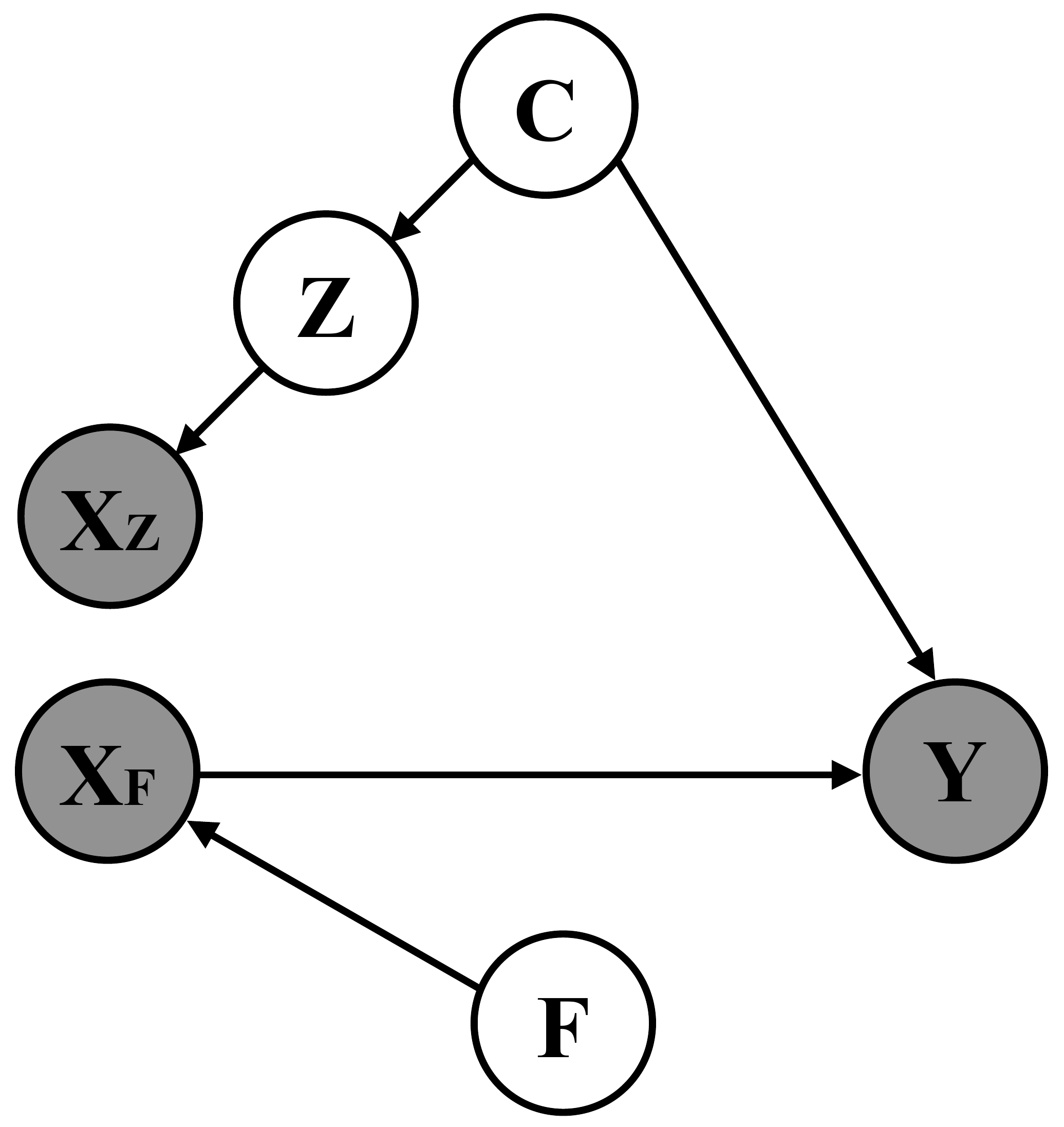}
  \caption{If we can disentangle the nuisances $X_Z$ and causal feature $X_F$ in the images, we get the above causal graph. We can identify causality by directly predicting $Y$ from $X_F$. We can measure the spurious correlations by predicting $Y$ from $X_Z$. However, disentangling $X$ into the nuisances $X_Z$ and causal feature $X_F$ is still an open challenge.}

  \label{fig:disentangled-graph}
\end{figure}

In our main paper, we mentioned that "$Z$ produces the nuisance features $X_Z$ in images. There is another variable $F$ that generates the core object features $X_F$, which together with $X_Z$ constructs the pixels of a natural image $X$. There is no direct arrow from $F$ to $Y$ since $Y \indep F | X$, i.e., image $X$ contains all the features for predicting $Y$. We can observe $X$ but not $X_Z$ or $Z_F$ separately. '' We show the corresponding causal graph in Figure \ref{fig:disentangled-graph}. If we can disentangle $X_Z$ and $X_F$, we can estimate the causal effect by predicting $Y$ from $X_F$. However, given that the image pixels we observe are combinations of $X_F$ and $X_Z$, we essentially use the causal graph in Figure \ref{fig:causalgraph}.

%% file: supp/source/causalbound.tex
\section{Calculating $P(x|z)$ For Causal Effect Bound}

\subsection{Causal Bound and Performance}
In Section 5.4, we empirically calculate the $P(x|z)$ for our causal effect bound. We describe the implementation details here.
Following Theorem \ref{th:2IV}, we measure the effectiveness of our intervention using $\log P(x|z)$ given that $\log$ is a monotonic function, which will not change the relative ranking of different $P(x|z)$. Our theorem suggests that a larger $\log P(x|z)$ results in a more effective intervention strategy for causal identification. To estimate $\log P(x|z)$, we sample $x$ from the ImageNet validation set, which is a widely used benchmark and approximates the true unknown distribution. Other open-world data can also be used here in future work. 

We estimate the likelihood to query validation set of images $x$ given our intervention $z$.  We denote the $i$-th query image from the validation set as $x_i$. The log likelihood of producing the test set data is $\log P(x|z) = \sum_{i}\log P(x_i|z) = \sum_{i}\sum_{x'_j} \log (P(x_i|x'_j)P(x'_j|z))$, where $x'$ is the data we generate in the training set. $P(x_i|x'_j)$ is calculated through the cross-entropy loss of the CNN features after  normalization. Since the generator is a deterministic network given the intervention, we have $P(x'_j|z)=1$. Thus $P(x_i|z) = \sum_{x'_j} P(x_i|x'_j)$.  We approximate the marginalization of $x'_j$ from generated dataset which is nearest to the query image. A larger value of $\log P(x|z)$ indicates the generated image under intervention $z$ is closer to the real data distribution, which according to Theorem \ref{th:2IV}, tightens the causal identification bound.

We randomly sample half of ObjectNet and ImageNet overlapping categories to calculate $\log P(x|z)$. For each category, we generate 1000 images for $x'_j$  and use all the 50 validation data for $x_i$. We use a pre-trained ResNet18 model for extracting the CNN features, where we use the features from the penultimate layer.

We first study the impact of $\log P(x|z)$ on the actual classifier generalization ability. We change the value of $\log P(x|z)$ using different interventional strengths. We set truncation value to be 0.5 for the experiments. 
We vary the strength of the intervention by controlling the scale of the randomness. 
For the listed datasets, we intervene with a different intervention scale and sample 2000 images for 1000 ImageNet categories each. Thus the intervention strength is controlled by how much we randomize the generation process: (1) For 'None' interventional data, we sample from the observational GAN. (2) For weak interventional data, for each category we sample 50 random seeds, and for each seed, we intervene on 4 randomly selected PCA directions from the top 20 PCA components, and generate 10 data points for each PCA transformation within $s \in [-5,5]$ uniformly. (3) For medium interventional data, we first sample 500 random noise vectors, then intervene for each noise with 2 randomly selected PCA directions among the top 20 PCA components. Then we randomly sample 2 images from $s \in [-7,7]$. (4)  For strong interventional data, we first sample 1000 random noise vectors, then intervene on each noise vector with 2 randomly selected PCA direction among the top 20 PCA components. Then we randomly sample one image from $s \in [-9,9]$.

As shown in Figure 5 in the main paper, $P(x|z)$ and the model performance increases as the intervention is strengthened.  It demonstrates that a proper intervention strategy that increases $P(x|z)$  increases the models' performance, which also matches our Theorem \ref{th:2IV}. Moreover, in practice, the intervened GAN achieved much higher accuracy than the GAN method on both ImageNet and ObjectNet test set. 

\subsection{Optimal Intervention Strategy}
We conduct an ablation study for hyperparameters for intervention strategy, where we fix all the other variables and change only the hyperparameters we studied. Our paper finally uses the best setup as our intervention strategy.

Our intervention strategy contains three components: the truncation value for the BigGAN generator, the number of PCA direction selected, and the step size for the intervention scale. For all intervention strategies, we fix the number of generated data to be 500 for each category.
For the interventional scale, we first construct a reference scale as a unit, then vary our intervention scale with relative percentage. We visualize each PCA component to specify a reference scale for interventional range, such that after the intervention in the given direction, the image still looks realistic. The intervention range is small for the largest PCA component, and large for the non-top PCA directions. For example, we use $[-3, 3]$ for topmost PCA, and $[-12,12]$ for 40-th PCA.  As we only visual check a subset of all the categories, this range value is just a rough estimate for our reference, not the true range. We refer to the intervention scale via the relative size to the reference intervention.

The setup for our three ablation study is as follows:

\textbf{Truncation Value} We use the top 60 PCA with the 100\% interventional scale to the reference size. We vary the truncation value from 0.125 to 1. We plot the calculated $\log P(x|z)$ value.

\textbf{Top K PCA value selected.} We use truncation 1 with 100\% of the intervention range for this experiment.  We experiment with intervention on 1 PCA to top 80 PCA intervention, and plot $\log P(x|z)$ value.

\textbf{Interventional Scale s. } We use truncation 1 with top 60 PCA component for this experiment.  We vary the interventional scale based on their relative number to our predefined reference value, where the scale is selected from 10\% to 130\%. We find 100\% to our reference intervention strength yield the highest value for  $\log P(x|z)$. 

The results are visualized in Figure 6 in the main paper. We select the hyperparameters that produce the highest $P(x|z)$ for our experiments.

%% file: supp/source/CorrelationAnalysis.tex
\section{Setup for Correlation Analysis}

\begin{figure*}
  \centering
  \includegraphics[width=1.0\textwidth]{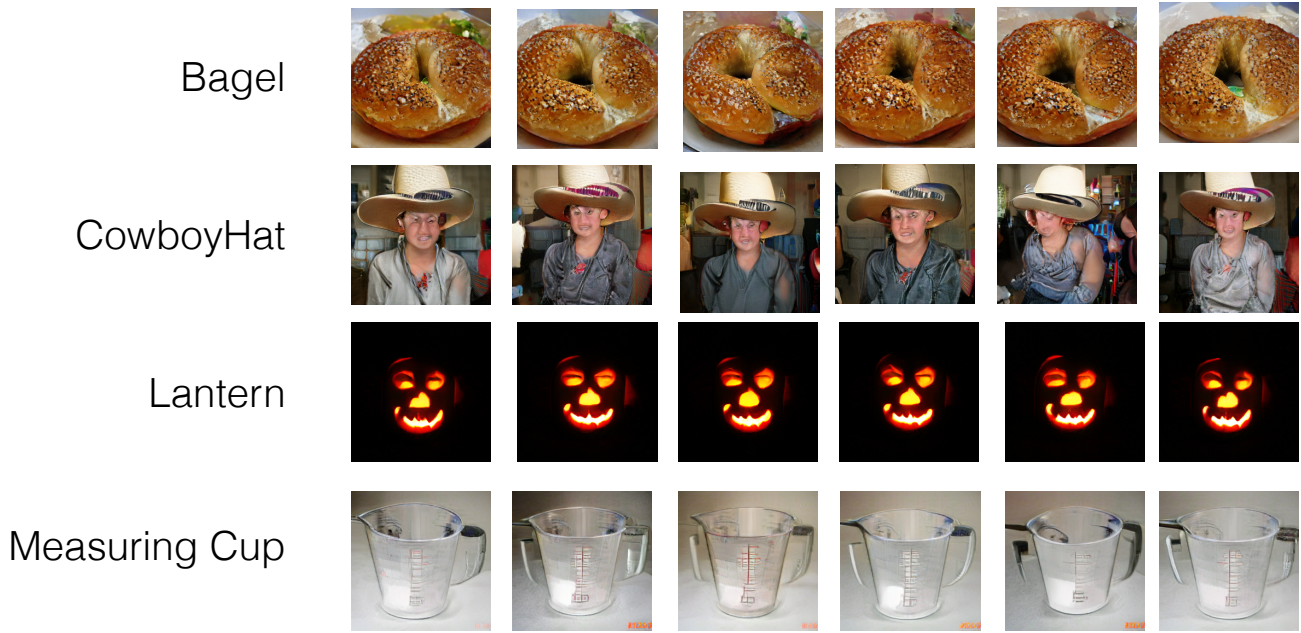}
  \vspace{-5mm}
  \caption{The exemplars generated from the BigGAN \cite{BigGAN} model with truncation 0.3. From here we intervene the image with interventions of scale up to 40 times of the truncation value. While the initial images may contain minor nuisances variations, it is negligible compared with our intervention strength.  We can thus infer the intervention value from the generated exemplars to the query images with reasonable precision (validation L1 error 0.36, while random model produces an error of 6).}
  \label{fig:exemplar}
\end{figure*}

In the main paper, we show that nuisance transformations such as viewpoints, backgrounds, and scene context automatically emerge in deep generative models. We hypothesise that there are spurious correlations between the nuisances and the object label. Given that nuisance factors are not the true cause for the label, simply predicting the label with images containing those spurious nuisances will result in non-causal models that are not robust.

We thus propose to quantify nuisances via generative models. We use the BigGAN \cite{BigGAN} \cite{BigGAN} trained on ImageNet to perform the empirical analysis. BigGAN is a conditional generative model. Given a category $\y$ and random noise vector $\h_0$, BigGAN generates an image $\x = G(\h_0,\y)$. The truncation factor in BigGAN controls the trade-off between quality and diversity. With small truncation 0.3, we find that  BigGAN not only generates highly realistic images but highly discriminative viewpoint with representative backgrounds. We show examples in Figure \ref{fig:exemplar}. We treat the generated images as the exemplar for the given category.

We then conduct interventions $\z$ to get the intervened noise vector $\h_0^*$, and the intervened image $\x^*= G(\h_0^*,\y)$, which corresponds to changing the viewpoints, backgrounds, and scene context of the exemplar. Our goal is to construct a dataset that contains images and their corresponding interventions. We aim to exhaust all the interventions $z$ on the given image, such that we can predict interventions well given an image. For each intervention $\z$, we intervene along a given PCA direction with a scale sampled randomly from a given range. We exhaust the top 60 PCA directions but remove 5 redundant transformations. This setup makes sure each PCA component correspond to some visual transformation that aligns with human perception. We focus on understanding the interpretable transformations, but adding more PCA directions produces stronger correlations. We specify the scale for each PCA direction to be within 12. For each category, we generate 10 random exemplars $h_0$. For each given $h_0$, we exhaust top 60 studied PCA directions and we sample 5 random intervention scale for each PCA direction.

We thus get data with both image $\x^*$ and the corresponding nuisance manipulation $\z$. We use 80\% of the data for training, and 20\% of the data for validation. We train a ResNet34 \cite{He_2016} model as backbone, which regresses the $z$ intervention scale value for each PCA direction given an input image $x$. We use L1 loss. Our training achieves an error of 0.009 and a validation error of 0.36, while the random guessing L1 error is around 5.6. Our trained regression model is category agnostic, thus the produced information does not contain category information. We further validate this by training an MLP model to predict the category information from the output nuisances, the model can only produce random guessing results. Thus we conclude our model can predict the nuisances well despite some error. 

We input ImageNet image to our trained nuisance predictor and output the corresponding nuisances $z$. We then train a 4 layer MLP model, with hidden units 512, to predict the ImageNet test label from the extracted nuisances factors. We find the learned MLP model can generalize well to the ImageNet test set, suggesting that nuisances factors are correlated with the label, which is shown in Figure 3 in the main paper. We also input data processed with mixup, stylize transfer, auto augmentation method, and GAN augmentation method, we then compute their corresponding correlations between the predicted nuisances and the labels. Results are shown in Figure 3 in the main paper, where we can see correlations are still there. However, after inputting our GenInt data, the correlations disappear, which suggests that our intervention removes the correlations.

%% file: supp/source/experiment.tex
\section{Experimental Setup and Ablation Study for Hyperparameters}

\begin{table*}[t]
\begin{center}
    \centering
    \small
    \begin{tabular}{ccccc|cc|cc}
         \toprule
          & Type & ImageNet & Interventional Data & Transferred Interventional & \multicolumn{2}{c|}{Std.\ Augmentation}  & \multicolumn{2}{c}{Add.\ Augmentation}\\
          & & $X$ &  $X_{int}$ &  Data $X_{itr}$ & top1   & top5  & top1 & top5\\
         \midrule  
          Baseline & A & \checkmark &  &   & 20.48\% & 40.64\% & 24.42\% & 44.39\% \\
         
          Ours & B&\checkmark &  \checkmark &     & 22.07\% & \textbf{41.94\%} &  25.71\% & 46.39\%\\
          Ours &C&\checkmark &  &  \checkmark  & 22.29\% & 41.76\% & 27.02\% & 47.51\% \\
          Ours & D&\checkmark &\checkmark  & \checkmark   & \textbf{22.34\%} & 41.65\% & \textbf{27.03\%} & \textbf{48.02\%} \\
         \bottomrule
    \end{tabular}
\end{center}
\caption{Ablation study for different generative interventions in our approach. On ResNet18, we experiment on training with different combinations of interventional data. The checkmark denote the data type is used in the training. We denote the model trained with only ImageNet data as `A,' the model trained with both ImageNet data and the interventional data $X_{int}$ as `B,' the model trained with ImageNet original and Transferred interventional data $X_{itr}$ as `C,' and the model trained with all data as type `D'. We show accuracy on the ObjectNet. With standard augmentation, simply training the original data with our interventional data $X_{int}$ improves performance on the ObjectNet and achieves the highest top 5 accuracy. By training on with both $X_{int}$ and transferred interventional data $X_{itr}$, the classifier achieved the best on top 1 accuracy. For training with additional augmentation, training on our interventional data further increases performance, which demonstrates that our approach is complementary to standard data augmentation methods. Together with additional augmentation, we show combined training on both $X_{int}$ and $X_{itr}$ achieves the best performance on both top 1 and top 5 accuracy.}  
\label{tab:objectnet}
\end{table*}

\subsection{Implementation details}

We used Nvidia GeForce RTX 2080Ti GPU with 11 GB memory for experiments. As shown in Table \ref{tab:objectnet}, we denote the model trained with only ImageNet data with as `A,' the model trained with both ImageNet data and the interventional data $X_{int}$ as `B,' the model trained with original ImageNet  and Transferred interventional data $X_{itr}$ as `C,' and the model trained with all data as type `D'. We also provide code in our supplementary.

\textbf{ResNet18 training details:} For all A,B,C,D model type, we train the model with batch size 256 for ImageNet $X$. For B, we use $\lambda_1=0.05$ and batch size 64 for interventional data $X_{int}$ term. For the type C model, we train the ImageNet model with transferred interventional data $X_{itr}$ with $\lambda_2=1$, where we both use a batch size of 256. We use the same parameter setup for both the training with standard augmentation and the additional augmentation. For type D model, we directly fine-tune the pre-trained GenInt Transfer model with $\lambda_1=0.05$ for additional augmentation and $\lambda_1=0.02$ for standard augmentation, with batch size 64 and $\lambda_2=1$, batch size 256. We will provide an ablation study for the batch size and $\lambda$ hyperparameters in the next subsection, which will show that our method constantly and robustly outperforms the baseline under a wide range of hyperparameters.

\textbf{ResNet152 training details:} We choose batch size as 256 for original ImageNet data $X$ for models A,B,C. For model B and model C, we use $\lambda_1=0.2, \lambda_2=0.2$ and batch size with 64. For model D, due to the GPU memory limitation, we reduce the batch size of original ImageNet data $X$ to be 192. We also use $\lambda_1=0.2, \lambda_2=0.2$ and batch size 64 for the interventional data. We finetune from model C.

\subsection{Ablation Study for Hyper-parameters}

\begin{table}[t]
\begin{center}
    \centering
    \scriptsize
    \begin{tabular}{lc|cc}
         \toprule
          Type & batch size for $X_{int}$ &  \multicolumn{2}{c}{Std.\ Augmentation}  \\
          &  & top1   & top5 \\
         \midrule  
          ImageNet baseline \cite{He_2016} & 0 & 20.48\% & 40.64\% \\
          \midrule
          GenInt (ours) & 32 & 21.89\% & 41.47\% \\
          GenInt (ours) & 64 & 22.07\% & \textbf{41.94\%} \\
          GenInt (ours) & 128 & \textbf{22.47\%} & 41.55\% \\
          GenInt (ours) & 256 & 21.92\% & 41.76\%  \\
         \bottomrule
    \end{tabular}
\end{center}
\caption{Ablation study for batch size in GenInt model. We investigate the effect of batch size when using our generative interventional data with the original data. We use $\lambda_1=0.05$ and change the batch size from 32 to 256. We observe that our method robustly outperforms the baseline under different hyperparameter setup.}  
\label{tab:objectnet-ablation-bs}
\end{table}

\begin{table}[t]
\begin{center}
    \centering
    \begin{tabular}{lc|cc}
         \toprule
          Type & $\lambda_1$ &  \multicolumn{2}{c}{Std.\ Augmentation}  \\
          &  & top1   & top5 \\
         \midrule  
          ImageNet baseline \cite{He_2016} & 0 & 20.48\% & 40.64\% \\
          \midrule
          GenInt (ours) & 0.02 & 21.86\% & 41.41\% \\
          GenInt (ours) & 0.05 & \textbf{22.07\% }& \textbf{41.94\%} \\
          GenInt (ours) & 0.2 & 21.69\% & 41.59\% \\
         \bottomrule
    \end{tabular}
\end{center}
\caption{Ablation study for the value of $\lambda$ in GenInt model. We investigate the effect of $\lambda_1$ when using our generative interventional data with the original data. We use batch size 64 and vary the value of $\lambda$ from 0 to 0.2. We observe our method robustly outperforms the baseline under different hyperparameter setup while performing best on $\lambda_1=0.05$.}  
\label{tab:objectnet-ablation-lam}
\end{table}

Besides the hyperparameters for intervention strategy (discussed in Section 5.4 in our main paper), our method also needs to specify the value for $\lambda$ and the corresponding batch size for the data. We offer ablation studies for the $\lambda$ and batch size here.

On ResNet18 models, we conduct an ablation study for batch size in Table \ref{tab:objectnet-ablation-bs}, and ablation study for the weight $\lambda_1$ in Table \ref{tab:objectnet-ablation-lam}.
Table \ref{tab:objectnet-ablation-bs} varies the batch size of $X_{int}$ data from 0 to 256, fixing $\lambda_1=0.05$. Our models perform the best between batch size 64 to 128. Table \ref{tab:objectnet-ablation-bs} varies the value for $\lambda_1$ from 0 to 0.2, with fixed batch size 64. Model performs the best at $\lambda_1=0.05$.
Table \ref{tab:objectnet-ablation-bs} and Table \ref{tab:objectnet-ablation-lam} demonstrate three major results: (1) our approach constantly and robustly outperforms the baseline under different setups, which suggests that one can achieve the state-of-the-art performance without much hyperparameter tuning  (2) the performance of our method will change as the hyperparameters change (3) after grid search for hyperparameters, we demonstrate even higher performance than the one we reported in our main paper (which is a randomly chosen number without many trials), which suggests our method can be improved given more computational resources and hyperparameter tuning.

%% file: supp/source/visualizations.tex
\section{Visualization}

\subsection{Model Visualization: Which regions are used by the model To make predictions}

By learning the right cause, we expect the models to learn to attend tightly to the spatial regions corresponding to the right object, and not spuriously correlated contextual regions. 
We use GradCAM to visualize all models' discriminative regions in Figure \ref{fig:gradcam-all}.
The results demonstrate that, our model not only outperforms the the-state-of-the-art ResNet152 model trained on ImageNet, but also outperforms other 4 established methods \cite{cubuk2018autoaugment, zhang2017mixup, imagenetbiased, CAS} for training robust classifiers with data augmentation. For example, for the first `Sunglasses' image, the baseline models attend to the spurious floor background and viewpoint, then mispredict the `Sunglasses' as `Toilet Tissue' and `Sandal'.  However, our method learns the right causal features, which ignores the spurious background and viewpoint information, leading to the right prediction. The same mispredictions happen for `Tie' and `Envelope' as well: given a wooden background, the state-of-the-art classifiers tend to predict based on the background and context, getting mispredictions such as `Cleaver' and `Sandal'. This suggests that, in addition to performance gains upon the established 5 baseline methods, our model predicts the right category for the right reasons. 

\begin{figure*}
  \centering
  \includegraphics[width=\textwidth]{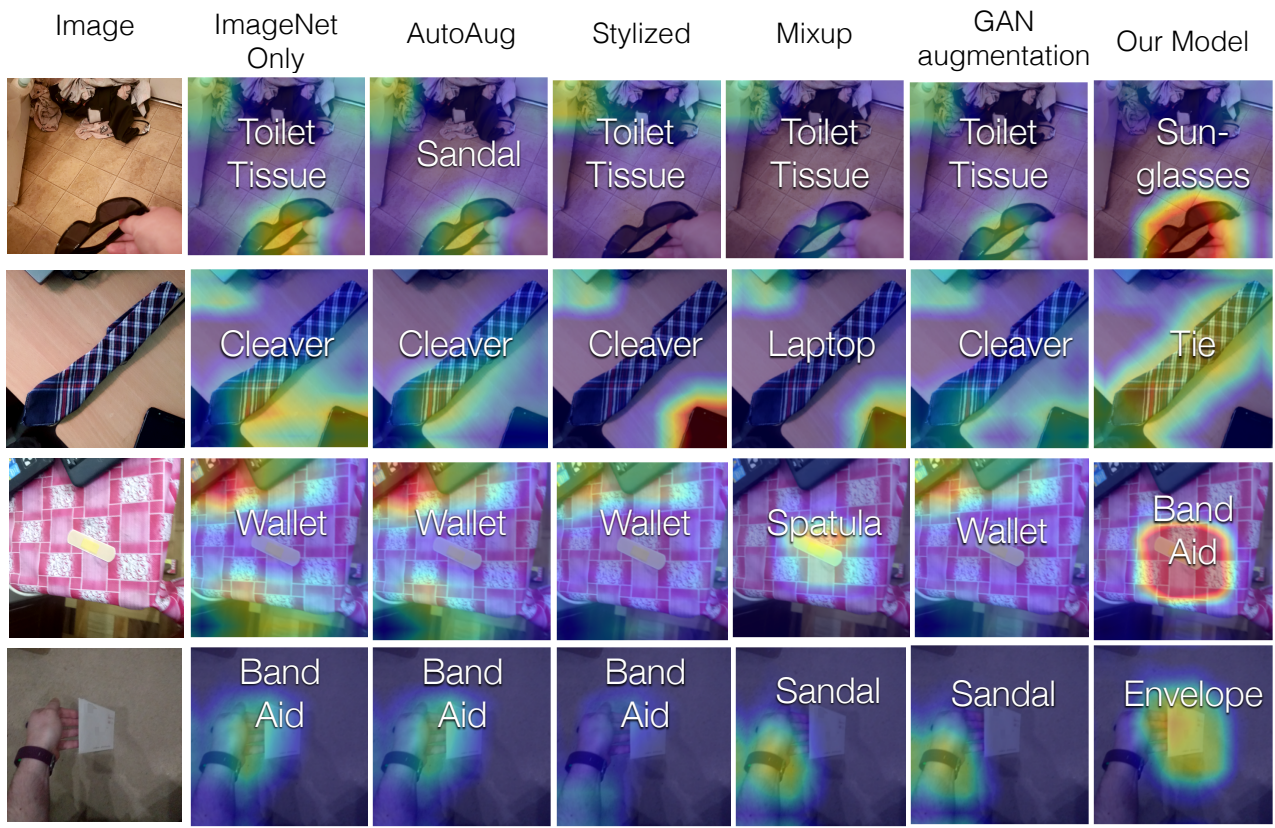}
  \caption{We visualize the input regions that are used to make predictions for 5 baseline models and our models. Blue indicates the model ignores the region for discrimination, while red indicates the region is the key for discriminative. The white text shows the model's top prediction. The baseline frequently latches onto spurious background context, resulting in wrong predictions. While the state-of-the-art data augmentation and robust learning methods cannot attend to the causal region, our model often predicts the right thing for the right reasons.}
  \label{fig:gradcam-all}
\end{figure*}


\subsection{Equivariance in Generative Models}

In Figure \ref{fig:equi}, we show more examples of steering the transformation of images generated using BigGAN. The major data augmentations are often restricted to a few types of traditional transformations, such as rotation and color jittering. Our intervention method can intervene on a larger number of nuisances with high-level transformations, such as stretch, which is orthogonal and complementary to existing augmentation strategies. Also, simply sampling from a generator will result in transformations that are spuriously correlated to the object category. With our generative intervention, we enable the generator to create transformations that relate less to the object category. Our interventional strategy increases the value of $P(x|z)$, which tightens the causal bound according to our theory.
Our generator sample a larger number of examples from the tail region of the distribution, which intuitively prevents the model from overfitting to certain biases.

\begin{figure*}
  \centering
  \includegraphics[width=1.0\textwidth]{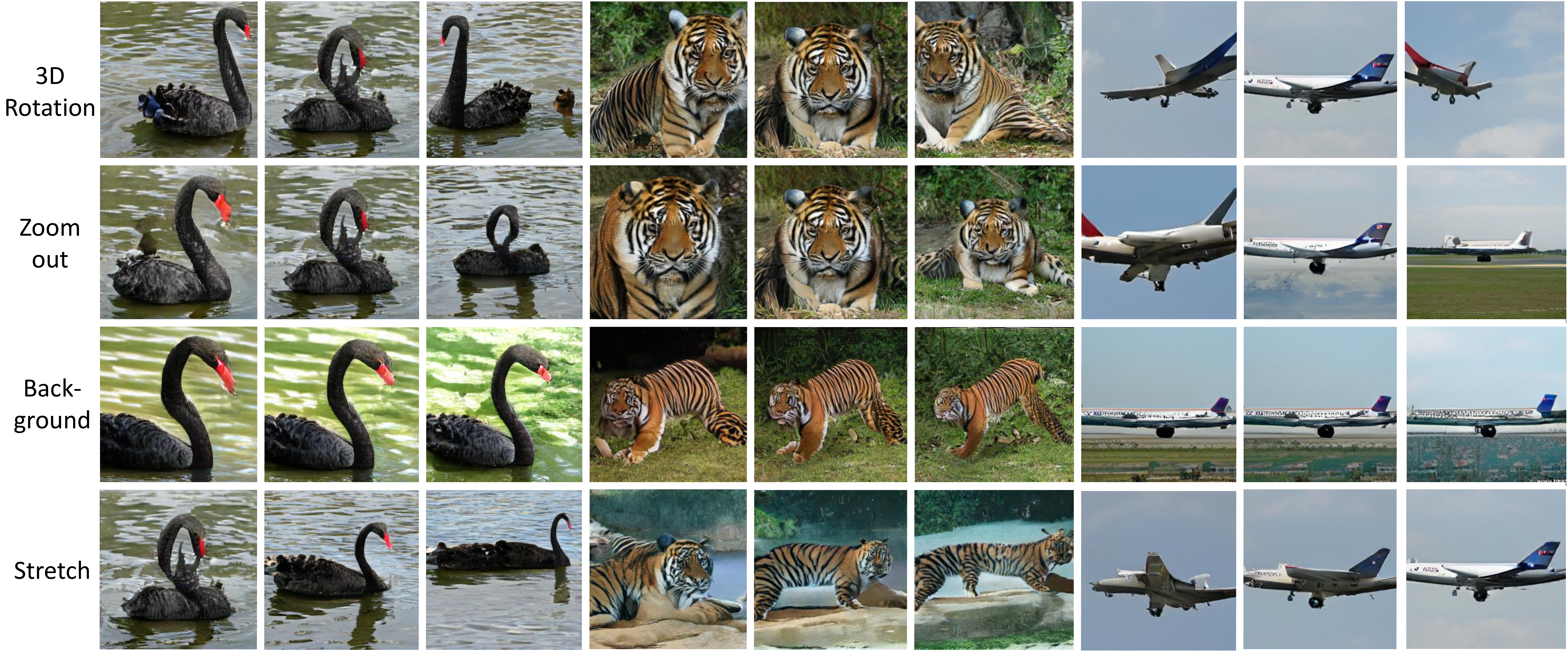}
  \vspace{-5mm}
  \caption{Generative adversarial networks are steerable, allowing us to manipulate images and approximate interventions. Since the representations are equivariant, the transformations transfer across categories. Each row in the figure presents images from the same intervention direction. }
  \label{fig:equi}
\end{figure*}

\subsection{Transferring generative interventions}

Eliminating spurious correlations can promote causality, producing robust classifiers. We propose to intervene on not only the generative data, but also the natural data. As we discussed in the main paper, projecting natural images to the latent space in generative models is challenging, yielding inferior results. Instead, we directly transfer the desirable generative interventions to natural data.  Leveraging the neural style transfer \cite{li2017demystifying, style_transfer}, we denote our generative interventional data as the `style' images, and the natural data as the `content' images.  We use the VGG model as the backbone for style transfer, where we use the features at 1st, 2nd, 3rd, 4th, 5th convolutional layer as the style and match the feature at 4th convolutional layer as content. We weight the style loss with 1000000. We use the LBFGS optimizer with default parameter setup in Pytorch. We apply update steps with number uniformly sampled from 20 to 70. The results for transferred interventions are visualized in Figure \ref{fig:transfer}. Our method can transfer the desirable interventions, such as the background and texture, to the target image, eliminating spurious correlations in the original natural data.

\begin{figure*}
  \centering
  \includegraphics[width=0.75\textwidth]{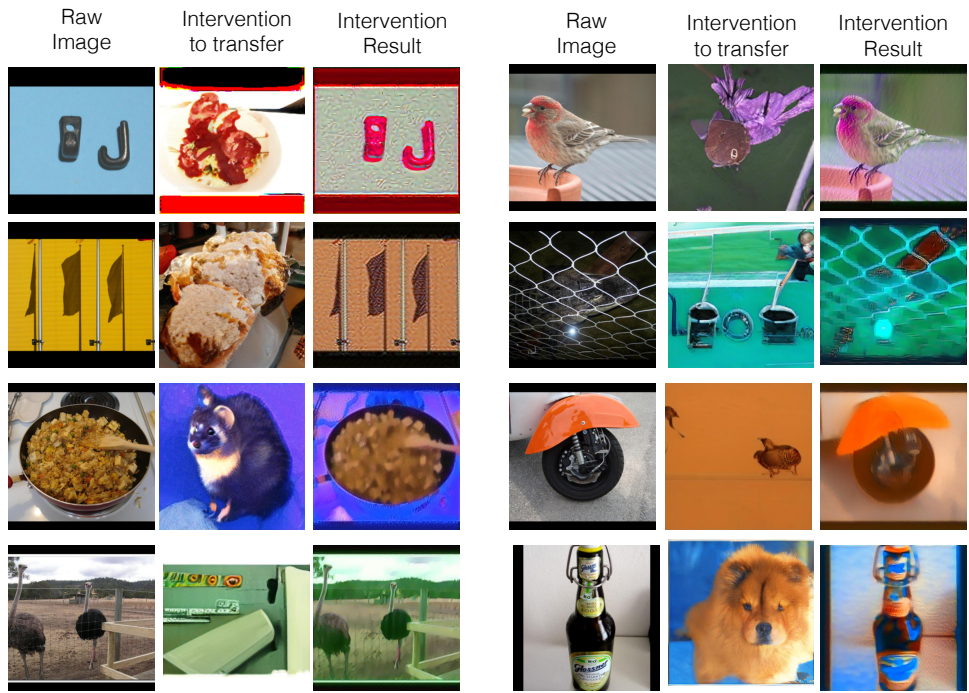}
  \caption{How to conduct interventions for natural images without knowing the corresponding latent code in a generator? We transfer the desirable intervention to our target images via the neural style transfer algorithm. The `Raw Image' column shows the images that we want to intervene on. The `Intervention to Transfer' column shows the result of applying our generative intervention to a random image generated by BigGAN, since intervention cannot exist itself, it needs to be realized on one image. For example, the first example has an intervention to whiten the background, where we realize it on a bagel image. Our goal is to apply this intervention to the natural image in the first column. The `Intervention Result' column shows the outcome of our transferred intervention, where our `Raw Image' is intervened with the intervention in the `Intervention to transfer' column. For the first image, we can see the background for the hook is whitened. We can thus perform interventions on the natural data where no latent code is found in the generative models. While the interventions here are mostly restricted to the background and scene context changes without viewpoint rotation, it can remove spurious correlations from the raw images to enforce a tighter causal bound. }
  \label{fig:transfer}
\end{figure*}